%% file: lpod-crprolog-tplp-0501.tex
\documentclass{new_tlp}
\usepackage{times}
\usepackage{helvet}
\usepackage{courier}
\usepackage{amsmath}
\usepackage{amssymb}
\usepackage{url}
\usepackage{graphicx}
\usepackage{color}
\usepackage{listings}

\newcommand{\cblu}{\color{blue}}

\def\crp{\rm \text{CR-Prolog}}
\def\crpt{{\rm \text{CR-Prolog}}_2}

\submitted {}
\revised {}
\accepted{}

\newtheorem{prop}{Proposition}
\newtheorem{lemma}{Lemma}
\newtheorem{thm}{Theorem}
\newtheorem{example}{Example}

\long\def\BOC#1\EOC{\message{(Commented text )}}
\long\def\BOCC#1\EOCC{\message{(Commented text )}}
\long\def\BOCCC#1\EOCCC{\message{(Commented text )}}
\long\def\optional#1{\empty}

\long\def\NBB#1{}

\def\bi{\begin{itemize}}
\def\ii{\item}
\def\ei{\end{itemize}}
\def\beq{\begin{equation}}
\def\eeq#1{\label{#1}\end{equation}}
\def\ba{\begin{array}}
\def\ea{\end{array}}
\def\i#1{\hbox{\it #1\/}}

\def\no{\i{not}}

\def\ar{\leftarrow}

\def\no{\i{not}}

\def\i#1{\hbox{\itshape #1\/}}

\def\proof{\noindent{\bf Proof}.\hspace{3mm}}

\begin{document}

\title{Translating LPOD and $\crpt$ into Standard Answer Set Programs} 

\author[Lee \& Yang]{Joohyung Lee and Zhun Yang \\
School of Computing, Informatics and Decision Systems Engineering \\
Arizona State University, Tempe, USA \\
\email{\{joolee, zyang90\}@asu.edu}}

\maketitle

\begin{abstract}
Logic Programs with Ordered Disjunction (LPOD) is an extension of standard answer set programs to handle preference using the construct of ordered disjunction,  and $\crpt$ is an extension of standard answer set programs with consistency restoring rules and LPOD-like ordered disjunction. We present reductions of each of these languages into the standard ASP language, which gives us an alternative way to understand the extensions in terms of the standard ASP language. 

\noindent
(The paper is under consideration for acceptance in TPLP.)
\end{abstract}

\vspace{-0.5em}

\section{Introduction} \label{sec:intro}

In answer set programming, each answer set encodes a solution to the  problem that is being modeled. There is often a need to express that one solution is preferable to another, so several extensions of answer set programs were made to express a qualitative preference over answer sets.
In Logic Programs with Ordered Disjunction (LPOD) \cite{brewka02logic}, this is done by introducing the construct of ordered disjunction in the head of a rule: $A\times B\ar\i{Body}$ intuitively means, when $\i{Body}$ is true, if possible then $A$, but if $A$ is not possible, then at least $B$.
Proposition~2 from \cite{brewka02logic} states that there is no reduction of LPOD to disjunctive logic programs \cite{gel91b} based on the fact that the answer sets of disjunctive logic programs are subset-minimal whereas LPOD answer sets are not necessarily so. However, this justification is limited to translations that preserve the underlying signature, and it remained an open question if it is possible to turn LPOD into the language of standard ASP such as ASP-Core 2 \cite{calimeri12aspcore2} by using auxiliary atoms. In this paper, we provide a positive answer to this question. 

We present a reduction of LPOD to standard answer set programs by compiling away ordered disjunctions. The translation gives us an alternative way to understand the semantics of LPOD in terms of the standard ASP language, and more generally, a method to express preference relations among answer sets. 
Instead of iterating the generator and the tester programs as in \cite{brewka02implementing}, our reduction is one-pass: the preferred answer sets can be computed by calling an answer set solver one time. 
 
It turns out that the translation idea is not restricted to LPOD but also applies to $\crpt$ \cite{balduccini04a-prolog}, which not only has a construct similar to ordered disjunction in LPOD but also inherits the construct of consistency-restoring rules---rules that can be added to make inconsistent programs to be consistent---from $\crp$ \cite{balduccini03logic}.
With some modifications to the LPOD translation, we show that $\crpt$ programs can also be turned into standard answer set programs by compiling away both ordered disjunctions and consistency-restoring rules.

\BOCC
In~\cite{brewka02implementing}, LPOD is implemented using {\sc smodels}. The implementation interleaves the execution of two programs--a generator which produces candidate answer sets and a tester which checks whether a given candidate answer set is maximally preferred or produces a more preferred candidate if it is not. An implementation of $\crp$ reported in~\cite{balduccini07cr-models} uses a similar algorithm. On the other hand, the reductions shown in this paper is one-pass and can be computed by calling an answer set solver one time without the need for iterating the generator and the tester.
\EOCC

The paper is organized as follows. Section~\ref{sec:lpod2asp} reviews  LPOD and presents a translation that turns LPOD into standard answer set programs. Section~\ref{sec:crp2asp} reviews $\crpt$ and presents a translation that turns $\crpt$ into standard answer set programs. The complete proofs are in the appendix.

\section{LPOD to ASP with Weak Constraints} \label{sec:lpod2asp}
\subsection{Review: LPOD} \label{subsec:review:lpod}

We review the definition of LPOD by~\citeN{brewka02logic}. As in that paper, for simplicity, we assume the underlying signature is propositional. 

\medskip
\noindent{\bf Syntax: }\ \ 
A (propositional) LPOD $\Pi$ is $\Pi_{reg}\cup \Pi_{od}$, where 
its {\em regular part} $\Pi_{reg}$ consists of usual ASP rules
$
\i{Head} \ar \i{Body}
$, 
and its {\em ordered disjunction part} $\Pi_{od}$ consists of {\em LPOD rules} of the form
\beq
    C^1\times\dots \times C^n \leftarrow \i{Body}  
\eeq{equ:lpodrule}
in which $C^i$ are atoms, $n$ is at least $2$, and $\i{Body}$ is a conjunction of atoms possibly preceded by $\no$.\footnote{In \cite{brewka02logic}, a usual ASP rule is viewed as a special case of a rule with ordered disjunction when $n=1$ but in this paper, we distinguish them. This simplifies the presentation of the translation and also allows us to consider LPOD that are more general than the original definition by allowing modern ASP constructs such as aggregates. }
Rule (\ref{equ:lpodrule}) intuitively says ``when $\i{Body}$ is true, if possible then $C^1$; if $C^1$ is not possible then $C^2$;  \dots; if all of $C^1,\dots, C^{n-1}$ are not possible then $C^n$.'' 

\medskip
\noindent{\bf Semantics: }\ \ 
For an LPOD rule \eqref{equ:lpodrule}, its {\em $i$-th option} ($i=1,\dots, n$) is defined as 
$$
    C^i\leftarrow \i{Body}, \no\ C^1, \dots, \no\ C^{i-1}.
$$

A {\em split program} of an LPOD $\Pi$ is obtained from $\Pi$ by replacing each rule in $\Pi_{od}$ by one of its options. A set $S$ of atoms is a {\em candidate answer set} of $\Pi$ if it is an answer set of a split program of~$\Pi$.

\begin{example} \label{ex:abcd}
(From~\cite{brewka02logic}) The following LPOD $\Pi_1$,
\[
\ba {rcl}
a \times b & \leftarrow & \no\ c \\
b \times c & \leftarrow & \no\ d,
\ea
\]
has four split programs:
\beq
\ba {lcl}
a \leftarrow \no\ c & \hspace{2cm} &  a \leftarrow \no\ c \\
b \leftarrow \no\ d &  & c \leftarrow \no\ d, \no\ b \\
&\\[-0.5em]
b \leftarrow \no\ c, \no\ a &  & b \leftarrow \no\ c, \no\ a \\
b \leftarrow \no\ d &   & c \leftarrow \no\ d, \no\ b .
\ea
\eeq{ex1-split}
Each of them has the following answer sets respectively, which are the candidate answer sets of~$\Pi_1$. 
\[
\ba {lcl}
\{ a,b \} & \hspace{2.5cm} & \{ c \} \\
&\\[-0.5em]
\{ b \} &  & \{ b \}, \{ c \}  .
\ea
\]
\end{example}


A candidate answer set $S$ of $\Pi$ is said to {\em satisfy} rule \eqref{equ:lpodrule} 
\bi
\ii to degree $1$ if $S$ does not satisfy $\i{Body}$, and
\ii to degree $j$ ($1 \leq j \leq n$) if $S$ satisfies $\i{Body}$ and $j = min\{ k \mid C^k \in S \}$. 
\ei


When $\Pi_{od}$ contains $m$ LPOD rules, the {\em satisfaction degree list} of a candidate answer set $S$ of $\Pi$ is $(d_1,\dots,d_m)$ where $d_i$ is the degree to which $S$ satisfies rule $i$ in $\Pi_{od}$.
For a candidate answer set $S$, let $S^i(\Pi)$ denote the set of rules in $\Pi_{od}$ satisfied by $S$ to degree $i$. For candidate answer sets $S_1$ and $S_2$ of $\Pi$, \citeN{brewka05preferences} introduces the following four preference criteria. 
\begin{enumerate}
\item  {\bf Cardinality-Preferred: } 
$S_1$ is {\em cardinality-preferred} to $S_2$ ($S_1 >^c S_2$) if there is a positive integer $i$ such that $|S^i_1(\Pi)| > |S^i_2(\Pi)|$, and $|S^j_1(\Pi)| = |S^j_2(\Pi)|$ for all $j<i$.

\item
{\bf Inclusion-Preferred: }
$S_1$ is {\em inclusion-preferred} to $S_2$ ($S_1 >^i S_2$) if there is a positive integer $i$ such that $S^i_2(\Pi) \subset S^i_1(\Pi)$, and $S_1^j(\Pi) = S_2^j(\Pi)$ for all $j < i$.

\item
{\bf Pareto-Preferred: }
$S_1$ is {\em Pareto-preferred} to $S_2$ ($S_1 >^p S_2$) if there is a rule that is satisfied to a lower degree in $S_1$ than in $S_2$, and there is no rule that is satisfied to a lower degree in $S_2$ than in $S_1$.

\item
{\bf Penalty-Sum-Preferred: }
$S_1$ is {\em penalty-sum-preferred} to $S_2$ ($S_1 >^{ps} S_2$) if the sum of the satisfaction degrees of all rules is smaller in $S_1$ than in $S_2$.

\end{enumerate}

A candidate answer set $S$ of $\Pi$ is a {\em k-preferred} ($k\in \{ c,i,p,ps \}$) {\em answer set} if there is no candidate answer set $S'$ of $\Pi$ such that $S' >^k S$.

\medskip
\noindent{\sl Example \ref{ex:abcd} (Continued) }\\ 
Recall that $\Pi_1$ has three candidate answer sets: $\{ a,b\}$, $\{ b \}$, and $\{ c \}$. Their satisfaction degree lists are  (1,1), (2,1), and (1,2), respectively.
One can check that $\{ a,b\}$ is the only preferred answer set according to any of the four preference criteria.

\begin{example} \label{ex:hotel}
To illustrate the difference among the four preference criteria, consider the following LPOD $\Pi_2$ about picking a hotel near the Grand Canyon. $hotel(1)$ is a 2 star hotel but is close to the Grand Canyon, $hotel(2)$ is a 3 star hotel and the distance is medium, and $hotel(3)$ is a 4 star hotel but is too far.

\begin{minipage}[c]{0.5\textwidth}
\[
\ba {rl}
& close \times med \times far \times tooFar \\
& star4 \times star3 \times star2 \\
& 1\{hotel(X): X=1..3\}1 \\
& \bot \leftarrow hotel(1),\ \no\ close \\
& \bot \leftarrow hotel(1),\ \no\ star2 \\
\ea
\]
\end{minipage}
\begin{minipage}[c]{0.5\textwidth}
\[
\ba {rl}
& \bot \leftarrow hotel(2),\ \no\ med \\
& \bot \leftarrow hotel(2),\ \no\ star3 \\
& \bot \leftarrow hotel(3),\ \no\ tooFar \\
& \bot \leftarrow hotel(3),\ \no\ star4   \\ \\
\ea
\]
\end{minipage}

\smallskip\noindent
$\Pi_2$ has $4 \times 3$ split programs but only the following three programs are consistent (The regular part of $\Pi_2$ is not listed).
\[
\ba {lcl}
close &  \hspace{1cm} & med \leftarrow \no\ close \\
star2 \leftarrow \no\ star4, \no\ star3 &   & star3 \leftarrow \no\ star4 \\ 
&\\[-0.5em]
tooFar \leftarrow \no\ close, \no\ med, \no\ far &\\
star4&
\ea
\]
The candidate answer sets of $\Pi_2$ and their satisfaction degree lists are
\[
\ba {ll}
S_1=\{ hotel(1), close, star2, \dots \}, (1, 3) & S_2=\{ hotel(2), med, star3, \dots \}, (2, 2) \\
S_3=\{ hotel(3), tooFar, star4, \dots \}, (4, 1) &
\ea
\]
By definition, the cardinality-preferred answer set is $S_1$, the inclusion-preferred answer sets are $S_1$ and $S_3$, the Pareto-preferred answer sets are $S_1$, $S_2$ and $S_3$, while the penalty-sum-preferred answer sets are $S_1$ and $S_2$.
\end{example}


\subsection{An Alternative Way to Generate Candidate Answer Sets: Assumption Programs} \label{subsec:assumptionP}

Before we describe the translation of LPOD into standard answer set programs, we consider an alternative way to generate candidate answer sets together with their ``assumption degrees," which serves as a basis of our translation.

Let $\Pi$ be an LPOD with $m$ LPOD rules. For an LPOD rule $i$ ($i \in \{ 1,\dots,m \}$)
\beq
C^1_i \times \dots \times C^{n_i}_i \ar \i{Body}_i \ ,
\eeq{eq:lpod}
its {\em $x$-th assumption} ($x\in\{0, \dots, n_i\}$), 
denoted by $O_i(x)$, is defined as the set of ASP rules
\begin{align}
body_i & ~~\leftarrow~~ \i{Body}_i \label{hard:lpod:body} \\
\bot & ~~\leftarrow~~ x = 0,\ body_i \label{hard:lpod:1} \\
\bot & ~~\leftarrow~~ x > 0,\ \no\ body_i \label{hard:lpod:2}\\
C^{j}_i & ~~\leftarrow~~ body_i,\ x=j & \text{(for $1\leq j \leq n_i$)} \label{hard:lpod:3} \\
\bot & ~~\leftarrow~~ body_i,\ x\ne j, \no\ C^1_i, \dots, \no\ C^{j-1}_i, \ C^j_i  & \text{(for $1\leq j \leq n_i$)} \label{hard:lpod:4}
\end{align}
where $\i{body}_i$ is a new, distinct  atom for each LPOD rule $i$.
Rules~\eqref{hard:lpod:body}---\eqref{hard:lpod:2} ensure that the body of \eqref{eq:lpod} is false iff $x=0$. 
Rule (\ref{hard:lpod:3}) represents that $C_i^x$ is true under the $x$-th assumption, and rule (\ref{hard:lpod:4}) ensures that all atoms $C^1_i, \dots, C_i^{x-1}$ are false. The last two rules together tells us that the first atom in $C_i^1, \dots, C_i^{n_i}$ that is true is $C_i^x$. 
%
The reason we call rules~\eqref{hard:lpod:body}---\eqref{hard:lpod:4} the $x$-th assumption is because they encode a certain assumption imposed on rule~\eqref{eq:lpod} in deriving each candidate answer set: 
$x=0$ assumes $\i{Body}_i$ is false, whereas $x>0$ assumes $\i{Body}_i$ is true and the $x$-th atom in the head is to be derived.

An {\em assumption program} of an LPOD $\Pi$ is obtained from $\Pi$ by replacing each rule in $\Pi_{od}$ by one of its assumptions.
If each LPOD rule $i$ is replaced by its $x_i$-th assumption, we call $(x_1,\dots,x_m)$ the {\em assumption degree list} of the assumption program. 

The following proposition asserts that the candidate answer sets can be obtained from assumption programs instead of split programs. 

\begin{prop} \label{prop:lpod2asp:hardsplit}
For any LPOD $\Pi$ of $\sigma$ and any set $S$ of atoms of $\sigma$,
$S$ is a candidate answer set of $\Pi$ iff $S\cup \{ body_i \mid$ $S$ satisfies the body of rule $i$ in $\Pi_{od} \}$ is an answer set of some assumption program of $\Pi$.
\end{prop}


\noindent{\sl Example \ref{ex:abcd} (Continued) }\ \ 
The assumptions for rule
$
  a \times b  \leftarrow  \no\ c ,
$
denoted by $O_1(X_1)$, and the assumptions for rule
$
  b \times c  \leftarrow  \no\ d ,
$
denoted by $O_2(X_2)$ are as follows, where $X_1$ and $X_2$ range over $\{0,1,2 \}$.
\[
{\small
\ba {lrcllrcl}
O_1(X_1): & body_1 &\!\!\leftarrow\!\!& \no\ c & O_2(X_2): & body_2 &\!\leftarrow\!& \no\ d \\
& \bot &\!\leftarrow\!& X_1 = 0, body_1 &  & \bot &\!\leftarrow\!& X_2 = 0, body_2 \\
& \bot &\!\leftarrow\!& X_1 > 0, \no\ body_1 &  & \bot &\!\leftarrow\!& X_2 > 0, \no\ body_2 \\
& a &\!\leftarrow\!& body_1, X_1\!=\!1 & & b &\!\leftarrow\!& body_2, X_2\!=\!1 \\
& b &\!\leftarrow\!& body_1, X_1\!=\!2 & & c &\!\leftarrow\!& body_2, X_2\!=\!2 \\
& \bot &\!\leftarrow\!& body_1, X_1\ne 1, a & & \bot &\!\leftarrow\!& body_2, X_2\ne 1, b \\
& \bot &\!\leftarrow\!& body_1, X_1\ne 2, \no\ a, b & & \bot &\!\leftarrow\!& body_2, X_2\ne 2, \no\ b, c
\ea
}
\]

$\Pi_1$ has 9 assumption programs, 
\[
{\small
\ba {llll}
O_1(0) \cup O_2(0) && ~O_1(0) \cup O_2(1) & \fbox{$O_1(0) \cup O_2(2)$}, \{ c \} \\
O_1(1) \cup O_2(0) && \fbox{$O_1(1) \cup O_2(1)$}, \{ a,b \} & ~O_1(1) \cup O_2(2) \\
O_1(2) \cup O_2(0) && \fbox{$O_1(2) \cup O_2(1)$}, \{ b \} & ~O_1(2) \cup O_2(2), 
\ea
}
\]
among which the three assumption programs in the boxes are consistent. Their answer sets are shown together. 

An advantage of considering assumption programs over split programs is that the satisfaction degrees---a basis of comparing the candidate answer sets---can be obtained from the assumption degrees with a minor modification (Section~\ref{sssec:cand}).
This is in part because each candidate answer set is obtained from only one assumption program whereas the same candidate answer set can be obtained from multiple split programs (e.g., $\{b\}$ in Example~\ref{ex:abcd}). 

%


\subsection{Turning LPOD into Standard Answer Set Programs} \label{subsec:lpod2asp}

We define a translation ${\sf lpod2asp}(\Pi)$ that turns an LPOD $\Pi$ into a standard answer set program. 

Let $\Pi$ be an LPOD of signature $\sigma$ where $\Pi_{od}$ contains $m$ propositional rules with ordered disjunction:
\begin{align}
&1: & C^{1}_1 \times \dots \times C^{n_1}_1 &\ \ \leftarrow\ \  \i{Body}_1 &&\nonumber \\
&&& \dots &&\label{program:lpod}\\
&m: & C^{1}_m \times \dots \times C^{n_m}_m &\ \ \leftarrow\ \  \i{Body}_m &&\nonumber
\end{align}
where 
$1,\dots, m$ are rule indices, and 
$n_i \geq 2$ for $1 \leq i \leq m$.


The first-order signature $\sigma'$ of ${\sf lpod2asp}(\Pi)$ 
contains $m$-ary predicate constant $a/m$ for each propositional constant $a$ of $\sigma$.
Besides, $\sigma'$ contains the following predicate constants not in $\sigma$: 
$ap/m$
(``assumption program''), $\i{degree}/(m+1)$, $body_i/m$  ($i \in \{1, \dots, m\}$), $\i{prf}/2$ (``preferred''), and $pAS/m$ (``preferred answer set''). Furthermore, $\sigma'$ contains the following predicate constants according to each preference criterion: 

\begin{itemize}
\item  for cardinality-preferred:  $card/3$, $equ2degree/3$, $prf2degree/3$
 
\item  for inclusion-preferred: $even/1$, $equ2degree/3$, $prf2degree/3$

\item  for Pareto-preferred:  $equ/2$
 
\item  for penalty-sum-preferred: $sum/2$.
\end{itemize}
 
\subsubsection{Generate Candidate Answer Sets}\label{sssec:cand}

The first part of the translation ${\sf lpod2asp}(\Pi)$ is to generate all candidate answer sets of $\Pi$ based on the notion of assumption programs.
We use the assumption degree list as a ``name space'' for each candidate answer set, so that we can compare them in a single answer set program.



\medskip
\noindent
{\bf 1.}\ \  
We use atom $ap(x_1,\dots,x_m)$ to denote the assumption program whose assumption degree list is $(x_1,\dots, x_m)$. We consider all consistent assumption programs by generating a maximal set of $ap(\cdot)$ atoms: 
$ap(x_1,\dots,x_m)$ is included in an optimal answer set \footnote{For programs containing weak constraints, an optimal answer set is defined by the penalty that comes from the weak constraints that are violated. \cite{calimeri12aspcore2}} iff the assumption program denoted by $ap(x_1,\dots,x_m)$ is consistent. 
\begin{flalign}
& \{ ap(X_1,\dots,X_m): \ X_1=0..n_1,\ \dots \ , X_m=0..n_m \}. && \label{equ:lpod:choice} \\
& :\sim ap(X_1,\dots,X_m).\ \ \left[ -1, X_1, \dots, X_m \right] && \label{equ:lpod:wc} 
\end{flalign}
Rule (\ref{equ:lpod:choice}) generates an arbitrary subset of $ap(\cdot)$ atoms, each of which records an assumption degree list.
Rule (\ref{equ:lpod:wc}) is a weak constraint that maximizes the number of $ap(\cdot)$ atoms by adding the penalty $-1$ for each true instance of $ap(X_1,\dots,X_m)$.
Together with the rules below, these rules ensure that we consider all assumption programs that are consistent and that no candidate answer sets are missed in computing preference relationship in the second part of the translation.

\medskip
\noindent
{\bf 2.}\ \  We extend each atom to include the assumption degrees $X_1,\dots,X_m$, and append atom $ap(X_1,\dots,X_m)$ in the bodies of rules. 
 
\begin{itemize}
\item
For each rule
$
\i{Head} \ar \i{Body}
$
in $\Pi_{reg}$, ${\sf lpod2asp}(\Pi)$ contains
\begin{align}
 \i{Head}(X_1, \dots, X_m) \leftarrow~~& ap(X_1,\dots,X_m), 
                     \i{Body}(X_1, \dots, X_m) 
\label{equ:lpod:regular}
\end{align}
where $\i{Head}(X_1,\dots,X_m)$ and $\i{Body}(X_1,\dots,X_m)$ are obtained from $\i{Head}$ and $\i{Body}$ by replacing each atom $A$ in them with $A(X_1,\dots,X_m)$. Each schematic variable $X_i$ ranges over $\{ 0,\dots,n_i \}$.

\item
For each rule
\[
C^1_i \times \dots \times C^{n_i}_i \ar \i{Body}_i 
\]
in $\Pi_{od}$, where $n\geq 2$, ${\sf lpod2asp}(\Pi)$ contains
\begin{align}
 body_i(X_1, \dots, X_m) \leftarrow~~ & ap(X_1,\dots,X_m),  
                                                              \i{Body}_i(X_1, \dots, X_m)
&&\label{equ:lpod:body} \\
\bot \leftarrow~~ & ap(X_1,\dots,X_m),\ X_i = 0,\ body_i(X_1, \dots, X_m) && \label{equ:lpod:xi=0} \\
\bot \leftarrow~~ & ap(X_1,\dots,X_m),\ X_i > 0,\ \no\ body_i(X_1, \dots, X_m). && \label{equ:lpod:xi!=0} 
\end{align}
And for $1\leq j\leq n_i$, ${\sf lpod2asp}(\Pi)$ contains
\beq
\ba l
C^j_i(X_1, \dots, X_m) \leftarrow~~ body_i(X_1, \dots, X_m), X_i = j. 
\ea 
\eeq{equ:lpod:ck} 
\beq 
\ba l
\bot\leftarrow~~  body_i(X_1, \dots, X_m), X_i \ne j,\    \\
~~~~~~~~~~~~~~\no\ C^1_i(X_1, \dots, X_m), \dots, \no\ C^{j-1}_i(X_1, \dots, X_m), C^j_i(X_1, \dots, X_m).
\ea
\eeq{equ:lpod:ckconstraint}
\end{itemize}


\medskip
\noindent
{\bf 3.}\ \ The satisfaction degree list can be obtained from the assumption degree list encoded in $ap(x_1,\dots,x_m)$ by changing $x_i$ to $1$ if it was $0$. For this,  ${\sf lpod2asp}(\Pi)$ contains
%
\begin{align}
& 1\{ \i{degree}(ap(X_1,\dots,X_m), D_1, \dots, D_m):\   D_1=1..n_1, \dots, D_m=1..n_m \}1 && \nonumber \\
&\hspace{8cm} \leftarrow ap(X_1,\dots,X_m). && \label{equ:lpod:onedegree}
\end{align}
and for $1\leq i \leq m$, ${\sf lpod2asp}(\Pi)$ contains
\begin{align}
& \bot ~\leftarrow~ \i{degree}(ap(X_1,\dots,X_m), D_1, \dots, D_m),\ X_i = 0,\ D_i\ne 1. && \label{equ:lpod:degreenotbody} \\
& \bot ~\leftarrow~ \i{degree}(ap(X_1,\dots,X_m), D_1, \dots, D_m),\ X_i >0,\ D_i\ne X_i.  && \label{equ:lpod:degreebody}
\end{align}
Since all answer sets of the same assumption program are associated with the same satisfaction degree list, we say an assumption program {\em satisfies} LPOD rule $i$ to degree $d$ if its answer sets satisfy the rule to degree $d$. 
Rule \eqref{equ:lpod:onedegree} reads ``for any assumption program $ap(x_1,\dots,x_m)$, it has exactly one assignment of satisfaction degrees $D_1,\dots, D_m$.'' 
Rules \eqref{equ:lpod:degreenotbody} and \eqref{equ:lpod:degreebody} say that the assumption program $ap(x_1,\dots,x_m)$ satisfies LPOD rule $i$ to degree 1 if $x_i = 0$ (in which case $\i{Body}_i$ is false) and to degree $x_i$ if $x_i > 0$ (in which case $\i{Body}_i$ is true).


\smallskip

Let us denote the set of rules \eqref{equ:lpod:choice}---\eqref{equ:lpod:degreebody}
by ${\sf lpod2asp}(\Pi)_{base}$. 
Observe that the atoms $a({\bf v})$ in the original signature $\sigma$ are in the form of $a({\bf v}, x_1,\dots,x_m)$ in the answer sets of ${\sf lpod2asp}(\Pi)_{base}$.
We define a way to retrieve the candidate answer set of $\Pi$ by removing $x_1,\dots,x_m$ as follows.
Let $S$ be an optimal answer set of ${\sf lpod2asp}(\Pi)_{base}$, and let 
\[
   shrink(S, x_1, \dots, x_m) \text{ be } 
   \{ a({\bf v}) \mid 
     a({\bf v}, x_1, \dots, x_m) \in S \text{ and } a({\bf v})\in \sigma  \}.
\]
If $S\models ap(x_1,\dots,x_m)$, we define the set $shrink(S, x_1, \dots, x_m)$ as a {\em candidate answer set on $\sigma$} of ${\sf lpod2asp}(\Pi)_{base}$.\footnote{We also apply this notation to the full translation ${\sf lpod2asp}(\Pi)$ and ${\sf crp2asp}(\Pi)$ below.}  
%


\BOCC
Let $\Pi$ be a translation (e.g., (${\sf lpod2asp}(\Pi)_{base}$) of $\Pi$; let $S$ be an optimal answer set of $\Pi'$; let $x_1, \dots, x_m$ be a list of integers such that $x_i \in \{ 0,\dots, n_i\}$; and let
\[
shrink(S, x_1, \dots, x_m) \text{ be } \{ a({\bf v}) \mid a({\bf v}, x_1, \dots, x_m) \in S \text{ and } a({\bf v})\in \sigma  \}.
\]
If $S \models ap(x_1,\dots,x_m)$, we define the set $shrink(S, x_1, \dots, x_m)$ as a {\em candidate answer set on $\sigma$} of $\Pi'$. 
\EOCC

The following proposition asserts the soundness of the translation ${\sf lpod2asp}(\Pi)_{base}$.

\begin{prop} \label{prop:lpod2asp:base}
The candidate answer sets of an LPOD $\Pi$ of signature $\sigma$ are exactly the candidate answer sets on $\sigma$ of ${\sf lpod2asp}(\Pi)_{base}$.
\end{prop}

\noindent{\sl Example \ref{ex:abcd} Continued: } 
The following is the encoding of ${\sf lpod2asp}(\Pi_1)_{base}$ in the input language of {\sc clingo}. 
{\small
\begin{lstlisting}
%%%% 1 %%%%
{ap(X1,X2): X1=0..2, X2=0..2}.             :~ ap(X1,X2). [-1, X1, X2]

%%%% 2 %%%%
% a*b <- not c.
body_1(X1,X2) :- ap(X1,X2), not c(X1,X2).
:- ap(X1,X2), X1=0, body_1(X1,X2).         :- ap(X1,X2), X1>0, not body_1(X1,X2).

a(X1,X2) :- body_1(X1,X2), X1=1.           b(X1,X2) :- body_1(X1,X2), X1=2.

:- body_1(X1,X2), X1!=1, a(X1,X2).
:- body_1(X1,X2), X1!=2, not a(X1,X2), b(X1,X2).

% b*c <- not d.
body_2(X1,X2) :- ap(X1,X2), not d(X1,X2).
:- ap(X1,X2), X2=0, body_2(X1,X2).         :- ap(X1,X2), X2>0, not body_2(X1,X2).

b(X1,X2) :- body_2(X1,X2), X2=1.           c(X1,X2) :- body_2(X1,X2), X2=2.

:- body_2(X1,X2), X2!=1, b(X1,X2).
:- body_2(X1,X2), X2!=2, not b(X1,X2), c(X1,X2).

%%%% 3 %%%%
1{degree(ap(X1,X2), D1, D2): D1=1..2, D2=1..2}1 :- ap(X1,X2).

:- degree(ap(X1,X2), D1, D2), X1=0, D1!=1.
:- degree(ap(X1,X2), D1, D2), X1>0, D1!=X1.

:- degree(ap(X1,X2), D1, D2), X2=0, D2!=1.
:- degree(ap(X1,X2), D1, D2), X2>0, D2!=X2.
\end{lstlisting}}

The optimal answer set $S$ of ${\sf lpod2asp}(\Pi_1)_{base}$ is 
\beq
\{ap(1,1),  a(1,1),  b(1,1), \dots, ap(2,1),  b(2,1),  \dots, ap(0,2),  c(0,2), \dots
\}
\eeq{ex1-cand}
%
($body_i(\cdot)$ and $degree(\cdot)$ atoms are not listed).
Since $S$ satisfies $ap(1,1)$, $ap(2,1)$, and $ap(0,2)$, the candidate answer sets on $\sigma$ of ${\sf lpod2asp}(\Pi_1)_{base}$ are
\[
\ba {lll}
shrink(S,1,1) = \{ a, b \}, & 
shrink(S,2,1) = \{ b \},  &
shrink(S,0,2) = \{ c \}
\ea
\]
which are exactly the candidate answer sets of $\Pi_{1}$.

\BOCC
Indeed, the first line of \eqref{ex1-cand} encodes the candidate answer set $\{A,B\}$, whose satisfying degrees are $(1,1)$. The second line encodes the candidate answer set $\{B\}$, whose satisfying degrees are $(2,1)$ and the third encodes the candidate answer set $\{C\}$, whose satisfying degrees are $(1,2)$. Unlike the four split programs \eqref{ex1-split}, there are three pairs of indices $(1,1), (2,1), (0,2)$ which are almost the same as the satisfaction degrees except that when the body is false, the index gets $0$ instead of $1$.
\EOCC

\subsubsection{Find Preferred Answer Sets}


The second part of the translation ${\sf lpod2asp}(\Pi)$ is to compare the candidate answer sets to find the  preferred answer sets.
For each preference criterion, ${\sf lpod2asp}(\Pi)$ contains the following rules respectively.
Below $maxdegree$ is $max\{n_i \mid i\in \{ 1,\dots, m \} \}$. 
\begin{enumerate}
\item[(a)]
{\bf Cardinality-Preferred: } For this criterion, ${\sf lpod2asp}(\Pi)$ contains the following rules.
{\small 
\begin{align}
card(P, X, N) ~\leftarrow~& \i{degree}(P, D_1, \dots, D_m), X=1..maxdegree,   && \nonumber \\
~~~~~~~~~~~~~~&~N=\{ D_1=X; \dots; D_m=X \}. && \label{equ:lpod:card:card} \\
 equ2degree(P_1,P_2,X) ~\leftarrow~& card(P_1,X,N), card(P_2,X,N), P_1\ne P_2. && \label{equ:lpod:card:equ2} \\
 prf2degree(P_1,P_2,X) ~\leftarrow~& card(P_1,X,N_1), card(P_2,X,N_2), N_1>N_2. && \label{equ:lpod:card:prf2} \\
 \i{prf}(P_1, P_2) ~\leftarrow~& X=0..maxdegree-1, prf2degree(P_1,P_2,X+1), && \nonumber \\
~~~~~~~~~~~~~~&~~ X\{equ2degree(P_1,P_2,Y): Y=1..X\}. && \label{equ:lpod:card:prf} \\
 pAS(X_1, \dots, X_m) ~\leftarrow~& ap(X_1, \dots, X_m), \{\i{prf}(P, ap(X_1, \dots, X_m))\}0. && \label{equ:lpod:card:pAS}
\end{align}
}
$P$, $P_1$, and $P_2$ denote assumption programs in the form of $ap(X_1, \dots, X_m)$. 
$card(P, X, N)$ is true if $P$ satisfies $N$ rules in $\Pi_{od}$ to degree $X$.
$equ2degree(P_1,P_2,X)$ is true if $P_1$ and $P_2$ have the same number of rules that are satisfied to degree $X$.
$prf2degree(P_1,P_2,X)$ is true if $P_1$ satisfies more rules to degree $X$ than $P_2$ does.
$\i{prf}(P_1,P_2)$ is true if $P_1$ is cardinality-preferred to $P_2$: $P_1$ satisfies more rules to degree $X+1$ than $P_2$ does whereas they satisfy the same number of rules up to degree $X$. 
%
Rule (\ref{equ:lpod:card:pAS}) reads as: given an assumption program represented by $ap(X_1, \dots, X_m)$, if we cannot find an assumption program $P$ that is more preferable, then the answer sets of $ap(X_1, \dots, X_m)$ are all preferred answer sets of $\Pi$. Note that $P$ in rule (\ref{equ:lpod:card:pAS}) is a local variable that ranges over all $ap(\cdot)$ atoms.

\smallskip
\item[(b)]
{\bf Inclusion-Preferred: } For this criterion, ${\sf lpod2asp}(\Pi)$ contains the following rules.
{\small 
\begin{align}
even(0; 2). \hspace{0.55cm} &~~~~ && \label{equ:lpod:incl:even} \\
equ2degree(P_1,P_2,X) ~\leftarrow~& P_1\ne P_2, X=1..maxdegree, && \nonumber \\
~~~~~~~~~~~~~~~&~ \i{degree}(P_1, D_{11}, \dots, D_{1m}), \i{degree}(P_2, D_{21}, \dots, D_{2m}),  && \nonumber \\
~~~~~~~~~~~~~~~&~ C_1 = \{ D_{11}=X; D_{21}=X \}, \dots, 
                  C_m = \{ D_{1m}=X; D_{2m}=X \}, && \nonumber \\
~~~~~~~~~~~~~~~&~ even(C_1), \dots, even(C_m). && \label{equ:lpod:incl:equ2} \\
 prf2degree(P_1,P_2,X) ~\leftarrow~& P_1\ne P_2, X=1..maxdegree, && \nonumber \\
~~~~~~~~~~~~~~~&~ \no\ equ2degree(P_1,P_2,X), && \nonumber \\
~~~~~~~~~~~~~~~&~ \i{degree}(P_1, D_{11}, \dots, D_{1m}), \i{degree}(P_2, D_{21}, \dots, D_{2m}),  && \nonumber \\
~~~~~~~~~~~~~~~&~ \{D_{11} \neq X; D_{21}=X\}1, \dots, \{D_{1m} \neq X; D_{2m}=X\}1. && \label{equ:lpod:incl:prf2} \\
 \i{prf}(P_1, P_2) ~\leftarrow~& X=0..maxdegree-1, prf2degree(P_1,P_2,X+1), && \nonumber \\
~~~~~~~~~~~~~~~&~ X\{equ2degree(P_1,P_2,Y): Y=1..X\}. && \label{equ:lpod:incl:prf} \\
 pAS(X_1, \dots, X_m) ~\leftarrow~& ap(X_1, \dots, X_m), \{\i{prf}(P, ap(X_1, \dots, X_m))\}0. && \label{equ:lpod:incl:pAS}
\end{align}
}
\noindent
where $\{ D_{11}=X; D_{21}=X \}$ counts the number of true atoms in this set, so it equals to 0 (or 2) when none (or both) of $D_{11}=X$ and $D_{21}=X$ are true; $\{D_{11} \neq X; D_{21}=X\}1$ means that the number of true atoms in this set must be smaller or equal to 1, which means that $D_{11}\neq X$ and $D_{21}=X$ cannot be true at the same time -- in other words, $D_{21}=X$ implies $D_{11}=X$.

\smallskip
\item[(c)] 
{\bf Pareto-Preferred: } For this criterion, ${\sf lpod2asp}(\Pi)$ contains the following rules.
{\small 
\begin{align}
 equ(P_1,P_2) ~\leftarrow~& \i{degree}(P_1, D_1, \dots, D_m), \i{degree}(P_2, D_1, \dots, D_m).  && \label{equ:lpod:pare:equ} \\
 \i{prf}(P_1, P_2) ~\leftarrow~& \i{degree}(P_1, D_{11}, \dots, D_{1m}), \i{degree}(P_2, D_{21}, \dots, D_{2m}), && \nonumber \\
~~~~~~~~~~~~~~~~ & \no\ equ(P_1,P_2), D_{11}\leq D_{21}, \dots, D_{1m}\leq D_{2m}. && \label{equ:lpod:pare:prf} \\
 pAS(X_1, \dots, X_m) ~\leftarrow~& ap(X_1, \dots, X_m), \{\i{prf}(P, ap(X_1, \dots, X_m))\}0. && \label{equ:lpod:pare:pAS}
\end{align}
}
where $equ(P_1,P_2)$ means that $P_1$ is equivalent to $P_2$ at all degrees.

\smallskip
\item[(d)]
{\bf Penalty-Sum-Preferred: } For this criterion, ${\sf lpod2asp}(\Pi)$ contains the following rules.
{\small
\begin{align}
 sum(P,N) ~\leftarrow~& \i{degree}(P, D_1, \dots, D_m), N=D_1+ \dots + D_m.  && \label{equ:lpod:pena:sum} \\
 \i{prf}(P_1, P_2) ~\leftarrow~& sum(P_1, N_1), sum(P_2, N_2), N_1 < N_2. && \label{equ:lpod:pena:prf} \\
 pAS(X_1, \dots, X_m) ~\leftarrow~& ap(X_1, \dots, X_m), \{\i{prf}(P, ap(X_1, \dots, X_m))\}0. && \label{equ:lpod:pena:pAS}
\end{align}
}
where $sum(P,N)$ means that the sum of $P$'s satisfaction degrees of all rules is $N$.
\end{enumerate}


If $S \models pAS(x_1, \dots, x_m)$, we define the set $shrink(S, x_1, \dots, x_m)$ to be a {\em preferred answer set on $\sigma$} of ${\sf lpod2asp}(\Pi)$.

a

The following 
theorem assert the soundness of the translation ${\sf lpod2asp}(\Pi)$.


\begin{thm} \label{thm:lpod2asp}
Under any of the four preference criteria, the candidate (preferred, respectively) answer sets of an LPOD $\Pi$ of signature $\sigma$ are exactly the candidate (preferred, respectively) answer sets on $\sigma$ of ${\sf lpod2asp}(\Pi)$.
\end{thm}

\noindent{\sl Example \ref{ex:hotel} Continued: } 
The first part of ${\sf lpod2asp}(\Pi_2)$ contains the following rules.
{\small
\begin{lstlisting}
#const maxdegree = 4.
  
%%%% 1 %%%%

{ap(X1,X2): X1=0..4, X2=0..3}.             :~ ap(X1,X2). [-1, X1, X2]

%%%% 2 %%%%

1{hotel(H,X1,X2): H=1..3}1 :- ap(X1,X2).
:- ap(X1,X2), hotel(1,X1,X2), not close(X1,X2).
:- ap(X1,X2), hotel(1,X1,X2), not star2(X1,X2).
:- ap(X1,X2), hotel(2,X1,X2), not med(X1,X2).
:- ap(X1,X2), hotel(2,X1,X2), not star3(X1,X2).
:- ap(X1,X2), hotel(3,X1,X2), not tooFar(X1,X2).
:- ap(X1,X2), hotel(3,X1,X2), not star4(X1,X2).

% close * med * far * tooFar.

body_1(X1,X2) :- ap(X1,X2).

:- ap(X1,X2), X1=0, body_1(X1,X2).         :- ap(X1,X2), X1>0, not body_1(X1,X2).

close(X1,X2) :- body_1(X1,X2), X1=1.        med(X1,X2) :- body_1(X1,X2), X1=2.
far(X1,X2) :- body_1(X1,X2), X1=3.          tooFar(X1,X2) :- body_1(X1,X2), X1=4.

:- body_1(X1,X2), X1!=1, close(X1,X2).
:- body_1(X1,X2), X1!=2, not close(X1,X2), med(X1,X2).
:- body_1(X1,X2), X1!=3, not close(X1,X2), not med(X1,X2), far(X1,X2).
:- body_1(X1,X2), X1!=4, not close(X1,X2), not med(X1,X2), not far(X1,X2),
   tooFar(X1,X2).

% star4 * star3 * star2.

body_2(X1,X2) :- ap(X1,X2).

:- ap(X1,X2), X2=0, body_2(X1,X2).          :- ap(X1,X2), X2>0, not body_2(X1,X2).

star4(X1,X2) :- body_2(X1,X2), X2=1.         star3(X1,X2) :- body_2(X1,X2), X2=2.
star2(X1,X2) :- body_2(X1,X2), X2=3.

:- body_2(X1,X2), X2!=1, star4(X1,X2).
:- body_2(X1,X2), X2!=2, not star4(X1,X2), star3(X1,X2).
:- body_2(X1,X2), X2!=3, not star4(X1,X2), not star3(X1,X2), star2(X1,X2).

%%%% 3 %%%%

1{degree(ap(X1,X2), D1, D2): D1=1..4, D2=1..3}1 :- ap(X1,X2).

:- degree(ap(X1,X2), D1, D2), X1=0, D1!=1.
:- degree(ap(X1,X2), D1, D2), X1>0, D1!=X1.

:- degree(ap(X1,X2), D1, D2), X2=0, D2!=1.
:- degree(ap(X1,X2), D1, D2), X2>0, D2!=X2.
\end{lstlisting}}

For the second part of the translation, ${\sf lpod2asp}(\Pi_2)$ contains one of the following sets of rules.

{\small
\begin{lstlisting}
%%%% a. Cardinality %%%%
card(P,X,N) :- degree(P,D1,D2), X=1..maxdegree, N={D1=X; D2=X}.
equ2degree(P1,P2,X) :- card(P1,X,N), card(P2,X,N), P1!=P2.
prf2degree(P1,P2,X) :- card(P1,X,N1), card(P2,X,N2), N1>N2.
prf(P1,P2) :- X=0..maxdegree-1, prf2degree(P1,P2,X+1), X{equ2degree(P1,P2,Y): Y=1..X}.
pAS(X1,X2) :- ap(X1,X2), {prf(P, ap(X1,X2))}0.
\end{lstlisting}}

{\small
\begin{lstlisting}
%%%% b. Inclusion %%%%
even(0;2).
equ2degree(P1,P2,X) :- P1!=P2, X=1..maxdegree, degree(P1,D11,D12), degree(P2,D21,D22), 
                         C1 = {D11=X; D21=X}, C2={D12=X; D22=X}, even(C1), even(C2).
prf2degree(P1,P2,X) :- P1!=P2, X=1..maxdegree, not equ2degree(P1,P2,X), 
                        degree(P1,D11,D12), degree(P2,D21,D22), 
                        {D11!=X; D21=X}1, {D12!=X; D22=X}1.
prf(P1,P2) :- X=0..maxdegree-1, prf2degree(P1,P2,X+1), X{equ2degree(P1,P2,Y): Y=1..X}.
pAS(X1,X2) :- ap(X1,X2), {prf(P, ap(X1,X2))}0.
\end{lstlisting}}

{\small
\begin{lstlisting}
%%%% c. Pareto %%%%
equ(P1,P2) :- degree(P1,D1,D2), degree(P2,D1,D2).
prf(P1,P2) :- degree(P1,D11,D12), degree(P2,D21,D22), not equ(P1,P2), 
              D11<=D21, D12<=D22.
pAS(X1,X2) :- ap(X1,X2), {prf(P, ap(X1,X2))}0.
\end{lstlisting}}

{\small
\begin{lstlisting}
%%%% d. Penalty-Sum %%%%
sum(P,N) :- degree(P,D1,D2), N=D1+D2.
prf(P1,P2) :- sum(P1,N1), sum(P2,N2), N1<N2.
pAS(X1,X2) :- ap(X1,X2), {prf(P, ap(X1,X2))}0.
\end{lstlisting}}


Note that each set of rules in the second part conservatively extends the answer set of the base program.  For example, the optimal answer set of ${\sf lpod2asp}(\Pi_1)$ under Penalty-Sum preference is the union of \eqref{ex1-cand} and 
$\{ sum(ap(0,2),3)$, $sum(ap(1,1),2)$, $sum(ap(2,1),3)$, $prf(ap(1,1),ap(0,2))$, $prf(ap(1,1),ap(2,1))$, $pAS(1,1) \}$, 
which indicates that $\{a,b\}$ is the preferred answer set.

The optimal answer set $S$ of ${\sf lpod2asp}(\Pi_2)$ under the cardinality preference is
\[
\ba {rllll}
\{pAS(1,3),& ap(1,3), &hotel(1,1,3), &close(1,3),  &star2(1,3) , \\
&ap(2,2), &hotel(2,2,2), &med(2,2), &star3(2,2), \\
&ap(4,1), &hotel(3,4,1), &tooFar(4,1), &star4(4,1), \dots \}
\ea
\]
Since $S$ satisfies $ap(1,3)$, $ap(2,2)$, and $ap(4,1)$, the candidate answer sets on $\sigma$ of ${\sf lpod2asp}(\Pi_2)$ are
\[
\ba {l}
shrink(S,1,3) = \{ hotel(1), close, star2 \}, \\
shrink(S,2,2) = \{ hotel(2), med, star3 \}, \\
shrink(S,4,1) = \{ hotel(3), tooFar, star4 \},
\ea
\]
which are exactly the candidate answer sets of $\Pi_2$.
Since $S$ satisfies $pAS(1,3)$, the preferred answer sets on $\sigma$ of ${\sf lpod2asp}(\Pi_2)$ is
$
shrink(S,1,3) = \{ hotel(1), close, star2 \}
$
which is exactly the cardinality-preferred answer set of $\Pi_2$.
Let 
\[
\ba {l}
pAS_1 = \{ pAS(1,3), hotel(1,1,3), close(1,3),star2(1,3) \}, \\
pAS_2 = \{ pAS(2,2), hotel(2,2,2), med(2,2), star3(2,2) \}, \\
pAS_3 = \{ pAS(4,1), hotel(3,4,1), tooFar(4,1), star4(4,1) \}.
\ea
\]
The optimal answer sets of ${\sf lpod2asp}(\Pi_2)$ under 4 criteria contain
\[
\ba {lcllc}
\text{cardinality-preferred:} & pAS_1 & & 
\text{inclusion-preferred:} & pAS_1 \cup pAS_3 \\
\text{Pareto-preferred:} & pAS_1 \cup pAS_2 \cup pAS_3 & &
\text{penalty-sum-preferred:} & pAS_1 \cup pAS_2
\ea
\]
which are in a 1-1 correspondence with the preferred answer sets of $\Pi_2$ under each of the four criteria respectively.

\section{$\crpt$ to ASP with Weak Constraints} \label{sec:crp2asp}

\subsection{Review: $\crpt$} \label{subsec:review:crp}

We review the definition of $\crpt$ from~\cite{bald03b}.

\noindent{\bf Syntax: }\ \ 
A (propositional) $\crpt$ program $\Pi$ consists of four kinds of rules:
\begin{flalign}
&\text{\em regular rule} &  & \i{Head} \leftarrow \i{Body} && \label{crp:r} \\
&\text{\em ordered rule} & i:~~ & C^1 \times \dots \times C^{n_i} \leftarrow \i{Body} && \label{crp:or} \\
&\text{\em cr-rule} & i:~~ & \i{Head} \stackrel{+}\leftarrow \i{Body} && \label{crp:cr} \\
&\text{\em ordered cr-rule} & i:~~ & C^1 \times \dots \times C^{n_i} \stackrel{+}\leftarrow \i{Body} && \label{crp:ocr}
\end{flalign}
where $\i{Head}\ar \i{Body}$ is a standard ASP rule, $i$ is the index of the rule, $C^j$ are atoms, and $n_i\geq 2$.
The intuitive meaning of an ordered disjunction $C^1 \times \dots \times C^{n_i}$ is similar to the one for LPOD. 
A cr-rule \eqref{crp:cr} or an ordered cr-rule \eqref{crp:ocr} is {\em applied} in $\Pi$ if it is treated as a usual ASP rule in $\Pi$ (by replacing $\stackrel{+}\ar$ with $\ar$); it is not applied if it is omitted in $\Pi$.
A cr-rule \eqref{crp:cr} or an ordered cr-rule \eqref{crp:ocr} is
applied only if the agent has no way to obtain a consistent set of beliefs using regular rules or ordered rules only.
By $\i{Head}(i)$ and $\i{Body}(i)$, we denote the head and the body of rule $i$.

\smallskip
\noindent{\bf Semantics: }\ \ 
The semantics of $\crpt$ is based on the transformation from a $\crpt$ program $\Pi$ of signature $\sigma$ into an answer set program $H_{\Pi}$, which is constructed as follows. 
The first-order signature of $H_{\Pi}$ is $\sigma \cup \{ \i{choice}/2, \i{appl}/1, \i{fired}/1, \i{isPreferred}/2 \}$, where $\i{choice}$ is a function constant,  $\i{appl}, \i{fired}, \i{isPreferred}$ are predicate constants not in $\sigma$.
\begin{enumerate}
\item

Let $R_{\Pi}$ be the set of rules obtained from $\Pi$ by replacing every cr-rule and ordered cr-rule of index $i$ with a rule:
\[
i: \ \ \i{Head}(i) \leftarrow \i{Body}(i), \i{appl}(i)
\]
where $\i{appl}(i)$ means rule $i$ is applied.
Notice that $R_{\Pi}$ contains only regular rules and ordered rules. 

$H_{\Pi}$ is then obtained from $R_{\Pi}$ by replacing every ordered rule of index $r$, where $\i{Head}(r)=C^1 \times \dots \times C^{n_i}$, with the following rules (for $1\leq j \leq n_i$):
\beq
\ba{l}
C^j \leftarrow \i{Body}(r), \i{appl}(\i{choice}(r, j)) \\
\i{fired}(r) \leftarrow \i{appl}(\i{choice}(r, j)) \\
\i{prefer}(\i{choice}(r, j), \i{choice}(r, j+1)) \hspace{1cm}  (j<n_i) \\
\bot\leftarrow \i{Body}(r), \no\ \i{fired}(r)
\ea
\eeq{h-pi}
where $\i{appl}(\i{choice}(r, j))$ means that the $j$-th atom in the ordered disjunction $\i{Head}(r)$ is chosen, i.e., $C^j$ is true if $\i{Head}(r)$ is true.

\item
$H_{\Pi}$ also contains the following set of rules:
\[
\ba {l}
\i{isPreferred}(R1,R2) \leftarrow \i{prefer}(R1,R2). \\
\i{isPreferred}(R1,R3) \leftarrow \i{prefer}(R1,R2), \i{isPreferred}(R2,R3). \\
\bot\leftarrow \i{isPreferred}(R,R). \\
\bot\leftarrow \i{appl}(R1), \i{appl}(R2), \i{isPreferred}(R1,R2).
\ea
\]
where $R1, R2, R3$ are schematic variables ranging over indices of cr-rules and ordered cr-rules in $\Pi$ as well as terms of the form $\i{choice}(\cdot)$.
\end{enumerate}

By $atoms(H_{\Pi}, \{\i{appl}\})$, we denote the set of atoms in $H_{\Pi}$ in the form of $\i{appl}(\cdot)$. A {\em generalized answer set} of $\Pi$ is an answer set of $H_{\Pi} \cup A$ where $A \subseteq atoms(H_{\Pi}, \{ \i{appl}\})$.

Let $S_1, S_2$ be generalized answer sets of $\Pi$. $S_1$ {\em dominates} $S_2$ if there exist $r_1$ and $r_2$ such that $\i{appl}(r_1) \in S_1$, $\i{appl}(r_2)\in S_2$, and $\i{isPreferred}(r_1,r_2) \in S_1 \cap S_2$.
Further, we say this domination is {\em rule-wise} if $r_1$ and $r_2$ are indices of two cr-rules; 
{\em atom-wise} if $r_1$ and $r_2$ are two terms of the form $\i{choice}(\cdot)$.
%
$S_1$ is a {\em candidate answer set} of $\Pi$ if there is no other generalized answer set that dominates $S_1$. 

The projection of $S_1$ onto $\sigma$ is a {\em preferred answer set} of $\Pi$ if $S_1$ is a candidate answer set of $\Pi$ and there is no other candidate answer set $S_2$ such that $S_2 \cap  atoms(H_{\Pi}, \{ \i{appl}\}) \subset S_1$.

\begin{example} \label{ex:crp2asp}
(From \cite{bald03b}) Consider the following $\crpt$ program $\Pi_3$:

\begin{minipage}[c]{0.3\textwidth}
\[
\ba {rl}
& q \leftarrow t. \\
& s \leftarrow t. \\ \\
\ea 
\]
\end{minipage}
\begin{minipage}[c]{0.3\textwidth}
\[
\ba {rl}
& p \leftarrow \no\ q. \\
& r \leftarrow \no\ s. \\
& \leftarrow p, r. \\
\ea
\]
\end{minipage}
\begin{minipage}[c]{0.3\textwidth}
\[
\ba {rl}
1:& t \stackrel{+}\leftarrow. \\
2:& q \times s \stackrel{+}\leftarrow. \\
\\
\ea
\]
\end{minipage}

\smallskip\noindent
which has 5 generalized answer sets (the atoms formed by $\i{isPreferred}$ or $\i{fired}$ are omitted)
\[
\ba {l}
S_1 = \{ q, s, t, \i{appl}(1), ~~\i{prefer}(\i{choice}(2, 1), \i{choice}(2, 2)) \} \\
S_2 = \{ q, r, \i{appl}(2), \i{appl}(\i{choice}(2, 1)), ~~\i{prefer}(\i{choice}(2, 1), \i{choice}(2, 2)) \} \\
S_3 = \{ p, s, \i{appl}(2), \i{appl}(\i{choice}(2, 2)), ~~\i{prefer}(\i{choice}(2, 1), \i{choice}(2, 2)) \} \\
S_4 = \{ q, s, t, \i{appl}(1), \i{appl}(2), \i{appl}(\i{choice}(2, 1)), ~~\i{prefer}(\i{choice}(2, 1), \i{choice}(2, 2)) \} \\
S_5 = \{ q, s, t, \i{appl}(1), \i{appl}(2), \i{appl}(\i{choice}(2, 2)), ~~\i{prefer}(\i{choice}(2, 1), \i{choice}(2, 2)) \}. \\
\ea
\]
Since $S_2$ (atom-wise) dominates $S_3$ and $S_5$, the candidate answer sets are $S_1$, $S_2$, and $S_4$. Since $S_1 \cap  atoms(H_{\Pi_3}, \{ \i{appl}\}) \subset S_4$, the preferred answer sets of $\Pi_3$ are the projections from $S_1$ or $S_2$ onto $\sigma$.

\end{example}

\subsection{Turning $\crpt$ into ASP with Weak Constraints} \label{subsec:crp2asp}

We define a translation ${\sf crp2asp}(\Pi)$ that turns a $\crpt$ program $\Pi$ into an answer set program with weak constraints.

Let $\Pi$ be a $\crpt$ program of signature $\sigma$, where its rules are rearranged such that the cr-rules are of indices $1,\dots, k$, the ordered cr-rules are of indices $k+1,\dots, l$, and the ordered rules are of indices $l+1,\dots, m$. 

For an ordered rule \eqref{crp:or} or an ordered cr-rule \eqref{crp:ocr}, its $i$-th {\em assumption}, where $i\in \{ 1, \dots, n_i\}$, is defined as $C^i \leftarrow \i{Body}$.
An {\em assumption program} $AP(x_1,\dots,x_m)$ of $\Pi$ whose {\em assumption degree list} is $(x_1,\dots,x_m)$ is obtained from $\Pi$ as follows ($x_i\in\{0,1\}$ if $i=1,\dots k$; $x_i\in \{0,\dots,n_i\}$ if $i=k\!+\!1,\dots, l$; $x_i\in\{1,\dots,n_i\}$ if $i=l\!+\!1,\dots,m$, where $n_i$ is the number of atoms in the head of rule $i$).


\begin{itemize}
\item
every regular rule \eqref{crp:r} is in $AP(x_1,\dots,x_m)$;
\item 
a cr-rule \eqref{crp:cr} is omitted if $x_i=0$, and is replaced by $\i{Head} \leftarrow \i{Body}$ if $x_i=1$;
\item 
an ordered cr-rule \eqref{crp:ocr} is omitted if $x_i=0$, and is replaced by its $x_i$-th assumption if $x_i>0$;
\item 
an ordered rule \eqref{crp:or} is replaced by its $x_i$-th assumption.

\end{itemize}
Besides, each assumption program $AP(x_1,\dots,x_m)$ contains
\[
\ba {l}
\i{isPreferred}(R1,R2) \leftarrow \i{prefer}(R1,R2). \\
\i{isPreferred}(R1,R3) \leftarrow \i{prefer}(R1,R2), \i{isPreferred}(R2,R3). \\
\leftarrow \i{isPreferred}(R,R). \\
\leftarrow x_{r_1}>0, x_{r_2}>0, \i{isPreferred}(r_1,r_2). \ \ \ \ \ \ (1 \leq r_1, r_2 \leq l)
\ea
\]

The generalized answer sets of $\Pi$ can be obtained from the answer sets of all the assumption programs of $\Pi$.

\begin{prop} \label{prop:crp:SP}
For any $\crpt$ program $\Pi$ of signature $\sigma$, a set $X$ of atoms is the projection of a generalized answer set of $\Pi$ onto $\sigma$ iff $X$ is the projection of an answer set of an assumption program of $\Pi$ onto $\sigma$.
\end{prop}
%

Let $\Pi_1$ and $\Pi_2$ be two assumption programs of $\Pi$. We say an answer set $S_1$ of $\Pi_1$ {\em dominates} an answer set $S_2$ of $\Pi_2$ if (i) there exists a rule $i$ in $\Pi$ that is replaced by its $j_1$-th assumption in $\Pi_1$, is replaced by its $j_2$-th assumption in $\Pi_2$, and $j_1 < j_2$; or (ii) there exist 2 rules $r_1,r_2$ in $\Pi$ such that $r_1$ is applied in $\Pi_1$, $r_2$ is applied in $\Pi_2$, and $prefer(r_1,r_2) \in S_1 \cap S_2$. 
Indeed, by Proposition \ref{prop:crp:SP}, $S_1$ dominates $S_2$ iff the corresponding generalized answer set of the former dominates that of the latter.

An answer set program with weak constraints ${\sf crp2asp}(\Pi)$ is obtained from $\Pi$ based on the notion of assumption programs as follows. 
%
The first-order signature $\sigma'$ of ${\sf crp2asp}(\Pi)$ contains $m$-ary predicate constant $a/m$ for each propositional constant $a$ of $\sigma$.
Besides, $\sigma'$ contains the following predicate constants not in $\sigma$: $ap/m$, $dominate/2$, $\i{isPreferred}/(m+2)$, $candidate/m$, $lessCrRulesApplied/2$, and $pAS/m$.
%

\medskip
\noindent
{\bf 1.}\ \ 
To consider a maximal set of consistent assumption programs, ${\sf crp2asp}(\Pi)$ contains 
%
\begin{flalign}
&\{ ap(X_1, \dots, X_{m}):  X_1=0..1,\ \dots \ , X_k=0..1,\ \
 X_{k+1} = 0..n_{k+1}, \dots, X_l=0..n_l, && \nonumber \\
&\hspace{2.8cm} X_{l+1} = 1..n_{l+1}, \dots, X_m=1..n_m \}. && \label{crp:1:choice} \\
&:\sim ap(X_1, \dots, X_{m}).  \left[ -1, X_1, \dots, X_{m} \right] && \label{crp:1:wc}
\end{flalign}
where $n_i$ is the number of atoms in $\i{Head}(i)$,
$ap(X_1, \dots, X_{p})$ denotes an assumption program obtained from $\Pi$. 



\medskip
\noindent
{\bf 2.}\ \ 
${\sf crp2asp}(\Pi)$ contains the following rules to construct all assumption programs $AP(x_1,\dots,x_m)$:
\begin{itemize}
\item
for each regular rule~~~
$
\i{Head} \leftarrow \i{Body}
$~~~
in $\Pi$, ${\sf crp2asp}(\Pi)$ contains
\begin{flalign}
&\i{Head}(X_1, \dots, X_m) \leftarrow ap(X_1, \dots, X_m), Body(X_1, \dots, X_m) && \label{crp:3:regular}
\end{flalign}

\item 
for each cr-rule~~~
$
i: ~~\i{Head}_i \stackrel{+}\leftarrow \i{Body}_i
$~~~
in $\Pi$, ${\sf crp2asp}(\Pi)$ contains
\begin{flalign}
&\i{Head}_i(X_1, \dots, X_m) \leftarrow ap(X_1, \dots, X_m), \i{Body}_i(X_1, \dots, X_m), X_i=1 && \label{crp:3:cr}
\end{flalign}

\item
for each ordered rule or ordered cr-rule ~~~$
i: ~~C^1_i \times \dots \times C^n_i \stackrel{(+)}\leftarrow \i{Body}_i
$~~~
in $\Pi$, 
for $1\leq j \leq n_i$, ${\sf crp2asp}(\Pi)$ contains
\begin{flalign}
&C^j_i(X_1, \dots, X_m) \leftarrow ap(X_1, \dots, X_m), \i{Body}_i(X_1, \dots, X_m), X_i=j && \label{crp:2}
\end{flalign}

\end{itemize}

\medskip
\noindent
{\bf 3.}\ \ 
To define $dominate$ in the semantics of $\crpt$, ${\sf crp2asp}(\Pi)$ contains the following rules.

\medskip\noindent
{\bf Atom-wise dominance:\ } 
Instead of using $\i{choice}(\cdot)$ terms and $\i{appl}(\i{choice}(\cdot))$ atoms in \eqref{h-pi}, we represent the atom wise dominance by comparing the assumption degrees.
For ordered cr-rules and ordered rules $i \in \{k+1, \dots m\}$, we  include
\begin{flalign}
&dominate(ap(X_1, \dots, X_m), ap(Y_1, \dots, Y_m)) \leftarrow && \nonumber \\
&~~~~~~~~~~~\hspace{3cm} ap(X_1, \dots, X_m), ap(Y_1, \dots, Y_m), 0<X_i, X_i<Y_i && \label{crp:4:atom} 
\end{flalign}

\noindent
{\bf rule-wise dominance:\ } {The following rules are included only when $\Pi$ contains an atom $\i{prefer}(\cdot)$. $r_1$ and $r_2$ ranges over $\{1,\dots, l\}$. }
{\small
\begin{flalign}
& \i{isPreferred}(R_1,R_2,X_1, \dots, X_m) \leftarrow \i{prefer}(R_1,R_2,X_1, \dots, X_m) && \label{crp:4:rule:1} \\
& \i{isPreferred}(R_1,R_3,X_1, \dots, X_m) \leftarrow \i{prefer}(R_1,R_2,X_1, \dots, X_m), && \nonumber \\
&~\hspace{5cm}~~~~~~~~~~~~~~~\i{isPreferred}(R_2,R_3,X_1, \dots, X_m) && \label{crp:4:rule:2} \\
& \leftarrow \i{isPreferred}(R,R,X_1, \dots, X_m) && \label{crp:4:rule:3} \\
& \leftarrow \i{isPreferred}(r_1,r_2,X_1, \dots, X_m), X_{r_1}>0, X_{r_2}>0 && \label{crp:4:rule:4} \\
&dominate(ap(X_1, \dots, X_m), ap(Y_1, \dots, Y_m)) \leftarrow 
    ap(X_1, \dots, X_m), ap(Y_1, \dots, Y_m), && \nonumber \\
&\hspace{1.5cm} \i{isPreferred}(r_1,r_2, X_1, \dots, X_m), \i{isPreferred}(r_1,r_2, Y_1, \dots, Y_m), 
 X_{r_1}>0, Y_{r_2}>0 && \label{crp:4:rule:5}
\end{flalign}
}

We say an assumption program $\Pi_1$ {\em dominates} an assumption program $\Pi_2$ if an answer set of $\Pi_1$ {dominates} an answer set of $\Pi_2$.
Indeed, our translation guarantees that if $\Pi_1$ dominates $\Pi_2$, all answer sets of $\Pi_1$ dominates any answer sets of $\Pi_2$. 
Rule \eqref{crp:4:atom} says that the assumption program $AP(x_1, \dots, x_m)$ dominates the assumption program $AP(y_1, \dots, y_m)$ if there exists a rule $i$ in $\Pi$ that is replaced by its $x_i$-th assumption in $AP(x_1, \dots, x_m)$, by its $y_i$-th assumption in $AP(y_1, \dots, y_m)$, and $x_i < y_i$.
Rules \eqref{crp:4:rule:1}, \eqref{crp:4:rule:2}, \eqref{crp:4:rule:3}, \eqref{crp:4:rule:4} are the set of rules in the semantics of $\crpt$ with the extended signature $\sigma'$.
Rule \eqref{crp:4:rule:5} says that $AP(x_1, \dots, x_m)$ dominates $AP(y_1, \dots, y_m)$ if 
$\i{isPreferred}(r_1,r_2)$ is true in both assumption programs while $r_1$ is applied in $AP(x_1, \dots, x_m)$ and $r_2$ is applied in $AP(y_1, \dots, y_m)$.

\medskip
\noindent
{\bf 4.}\ \ 
To define candidate answer sets in the semantics of $\crpt$, ${\sf crp2asp}(\Pi)$ contains
{\small
\beq 
candidate(X_1, \dots, X_m) \leftarrow ap(X_1, \dots, X_m), \{ dominate(P, ap(X_1, \dots, X_m)) \}0 
\eeq{crp:5}
}
Rule \eqref{crp:5} says that the answer sets of $AP(x_1, \dots, x_m)$ are candidate answer sets if there does not exist an assumption program $P$ that dominates $AP(x_1, \dots, x_m)$.

\medskip
\noindent
{\bf 5.}\ \ 
To define the preference between two candidate answer sets and find preferred answer sets, ${\sf crp2asp}(\Pi)$ contains
{\small
\begin{flalign}
&lessCrRulesApplied(ap(X_1, \dots, X_m), ap(Y_1, \dots, Y_m)) \leftarrow && \nonumber\\
&~~~~~~~~~~~~~~~~candidate(X_1, \dots, X_m), candidate(Y_1, \dots, Y_m), && \nonumber \\
&~~~~~~~~~~~~~~~~1\{X_1\neq Y_1; \dots ; X_m\neq Y_m\}, X_1\leq Y_1, \dots , X_m\leq Y_m  && \label{crp:6} \\
&pAS(X_1, \dots, X_m) \leftarrow candidate(X_1, \dots, X_m), \{ lessCrRulesApplied(P, ap(X_1, \dots, X_m)) \}0 && \label{crp:7}
\end{flalign}
}
Rule \eqref{crp:6} says that for any different assumption programs $AP(x_1, \dots, x_m)$ and $AP(y_1, \dots, y_m)$ whose answer sets are candidate answer sets, if all the choices in $AP(x_1, \dots, x_m)$ is not worse than
\footnote{
i.e., for any rule $i$ in $\Pi$, if it is applied in $AP(x_1, \dots, x_m)$, it must be applied in $AP(y_1, \dots, y_m)$; if it is replaced by its $x_i$-th assumption in $AP(x_1, \dots, x_m)$, it must be replaced by its $y_i$-th assumption in $AP(y_1, \dots, y_m)$ and $x_i \leq y_i$
}
those in $AP(y_1, \dots, y_m)$, then the former must apply less cr-rules or ordered cr-rules than the latter.
Rule \eqref{crp:7} says that the answer sets of $AP(x_1, \dots, x_m)$ are preferred answer sets if these answer sets are candidate answer sets and there does not exist an assumption program $P$ that applies less cr-rules than $AP(x_1, \dots, x_m)$.


Let $S$ be an optimal answer set of ${\sf crp2asp}(\Pi)$; $x_1, \dots, x_m$ be a list of integers.
If $S \models ap(x_1, \dots, x_m)$, we define the set $shrink(S, x_1, \dots, x_m)$ as a {\em generalized answer set on $\sigma$} of ${\sf crp2asp}(\Pi)$; if $S \models candidate(x_1, \dots, x_m)$, we define the set $shrink(S, x_1, \dots, x_m)$ as a {\em candidate answer set on $\sigma$} of ${\sf crp2asp}(\Pi)$; if $S \models pAS(x_1, \dots, x_p)$, we define the set $shrink(S, x_1, \dots, x_p)$ as a {\em preferred answer set on $\sigma$} of ${\sf crp2asp}(\Pi)$. 

\BOCC
\begin{prop} \label{prop:crp2asp:gAS}
The projections of the generalized answer sets of a $\crpt$ program $\Pi$ onto its signature $\sigma$ are exactly the generalized answer sets on $\sigma$ of ${\sf crp2asp}(\Pi)$.
\end{prop}

\begin{prop} \label{prop:crp2asp:cAS}
The projections of the candidate answer sets of a $\crpt$ program $\Pi$ onto its signature $\sigma$ are exactly the candidate answer sets on $\sigma$ of ${\sf crp2asp}(\Pi)$.
\end{prop}
\EOCC

\begin{thm} \label{thm:crp2asp} For any $\crpt$ program $\Pi$ of signature $\sigma$,
(a)  the projections of the generalized answer sets of $\Pi$ onto $\sigma$ are exactly the generalized answer sets on $\sigma$ of ${\sf crp2asp}(\Pi)$.
(b)  the projections of the candidate answer sets of $\Pi$ onto $\sigma$ are exactly the candidate answer sets on $\sigma$ of ${\sf crp2asp}(\Pi)$.
(c)  the preferred answer sets of $\Pi$ are exactly the preferred answer sets on $\sigma$ of ${\sf crp2asp}(\Pi)$.
\end{thm}

\smallskip
\noindent{\sl Example~\ref{ex:crp2asp} Continued: } 
\BOCC
Given a $\crpt$ program $\Pi$:
\[
\ba {rl}
& q \leftarrow t. \\
& s \leftarrow t. \\
& p \leftarrow \no\ q. \\
& r \leftarrow \no\ s. \\
& \leftarrow p, r. \\
&\\
1:& t \stackrel{+}\leftarrow. \\
2:& q \times s \stackrel{+}\leftarrow. \\
\ea
\]
\EOCC
The translated ASP program ${\sf crp2asp}(\Pi_3)$ is

{\small
\begin{lstlisting}
%%%% 1 %%%%
{ap(X1,X2): X1=0..1, X2=0..2}.             :~ ap(X1,X2). [-1,X1,X2]

%%%% 2 %%%%
q(X1,X2) :- ap(X1,X2), t(X1,X2).           s(X1,X2) :- ap(X1,X2), t(X1,X2).
p(X1,X2) :- ap(X1,X2), not q(X1,X2).       r(X1,X2) :- ap(X1,X2), not s(X1,X2).
:- ap(X1,X2), p(X1,X2), r(X1,X2).

% 1: t <+-.
t(X1,X2) :- ap(X1,X2), X1=1.

% 2: q*s <+-.
q(X1,X2) :- ap(X1,X2), X2=1.               s(X1,X2) :- ap(X1,X2), X2=2.

%%%% 3 %%%%
dominate(ap(X1,X2), ap(Y1,Y2)) :- ap(X1,X2), ap(Y1,Y2), 0<X1, X1<Y1.
dominate(ap(X1,X2), ap(Y1,Y2)) :- ap(X1,X2), ap(Y1,Y2), 0<X2, X2<Y2.

%%%% 4 %%%%
candidate(X1,X2) :- ap(X1,X2), {dominate(P,ap(X1,X2))}0.

%%%% 5 %%%%
lessCrRulesApplied(ap(X1,X2), ap(Y1,Y2)) :- candidate(X1,X2), candidate(Y1,Y2), 
        1{X1!=Y1;X2!=Y2}, X1<=Y1, X2<=Y2.
pAS(X1,X2) :- candidate(X1,X2), {lessCrRulesApplied(P,ap(X1,X2))}0.
\end{lstlisting}}

The optimal answer set $S$ of ${\sf crp2asp}(\Pi_3)$ is 
\[
\ba {rccl}
\{pAS(1,0), & candidate(1,0), & ap(1,0), & t(1,0), q(1,0), s(1,0), \\
pAS(0,1), & candidate(0,1), & ap(0,1), & q(0,1), r(0,1), \\
& & ap(0,2), & p(0,2), s(0,2), \\
& candidate(1,1), & ap(1,1), & t(1,1), q(1,1), s(1,1), \\
& & ap(1,2), & t(1,2), q(1,2), s(1,2), \dots 
\}.
\ea
\]
Since $S$ satisfies $ap(1,0)$, $ap(0,1)$, $ap(0,2)$, $ap(1,1)$, $ap(1,2)$, the generalized answer sets on $\sigma$ of ${\sf crp2asp}(\Pi_3)$ are

\begin{minipage}[c]{0.35\textwidth}
\[
\ba {l}
shrink(S,1,0) = \{ t,q,s \} \\
shrink(S,0,1) = \{ q,r \} \\
shrink(S,0,2) = \{ p,s \} \\
\ea
\]
\end{minipage}
\begin{minipage}[c]{0.5\textwidth}
\[
\ba{l}
shrink(S,1,1) = \{ t,q,s \} \\
shrink(S,1,2) = \{ t,q,s \} \\
\ea
\]
\end{minipage}

\smallskip\noindent
which are exactly the projections of the generalized answer sets of $\Pi_3$ onto $\sigma$.
Similarly, we observe that the candidate (preferred, respectively) answer sets on $\sigma$ of ${\sf crp2asp}(\Pi_3)$ are exactly the projections of the candidate (preferred, respectively) answer sets of $\Pi_3$ onto $\sigma$.

Furthermore, let $\Pi_3' = \Pi_3 \cup \{ \i{prefer}(2,1). \}$. The translation ${\sf crp2asp}(\Pi_3')$ is ${\sf crp2asp}(\Pi_3) \cup R$, where $R$ is the set of the following rules:

{\small
\begin{lstlisting}
%%%% 2 %%%%
prefer(2,1,X1,X2) :- ap(X1,X2).

%%%% 3 %%%%
isPreferred(R1,R2,X1,X2) :- prefer(R1,R2,X1,X2).
isPreferred(R1,R3,X1,X2) :- prefer(R1,R2,X1,X2), isPreferred(R2,R3,X1,X2).
:- isPreferred(R,R,X1,X2).
:- isPreferred(2,1,X1,X2), X2>0, X1>0.

dominate(ap(X1,X2), ap(Y1,Y2)) :- ap(X1,X2), ap(Y1,Y2), 
        isPreferred(2,1,X1,X2), isPreferred(2,1,Y1,Y2), X2>0, Y1>0.
\end{lstlisting}}

The optimal answer set $S$ of ${\sf crp2asp}(\Pi_3')$ is 
\[
\ba {lccl}
\{ &  & ap(1,0), & t(1,0), q(1,0), s(1,0), \\
~~pAS(0,1), & candidate(0,1), & ap(0,1), & q(0,1), r(0,1), \\
& & ap(0,2), & p(0,2), s(0,2), \dots 
\}
\ea
\]
and it is easy to check that the generalized (/candidate/preferred) answer sets on $\sigma$ of ${\sf crp2asp}(\Pi_3')$ are exactly the projections of the generalized (/candidate/preferred) answer sets of $\Pi_3'$ onto $\sigma$.

\BOCC
\subsection{Modification to Guarantee Consistency}

Based on our translation, we can test on any preference semantics by writing them down in ASP.
\EOCC


\section{Related Work and Conclusion} 


We presented reductions of LPOD and $\crpt$ into the standard ASP language, which explains the new constructs for preference handling in terms of the standard ASP language. The one-pass translations are theoretically interesting. They may be a useful tool for studying the mathematical properties of LPOD and $\crpt$ programs by reducing them to more well-known  properties of standard answer set programs.  
Both translations are ``almost'' modular in the sense that the  translations are rule-by-rule but the argument of each atom representing the assumption degrees may need to be expanded when new rules are added. 

%

However, the direct implementations may not lead to effective implementations. The size of ${\sf lpod2asp}(\Pi)$ and ${\sf crp2asp}(\Pi)$ after grounding could be exponential to the size of the non-regular rules in $\Pi$. This is because these translations compare all possible assumption programs whose number is exponential to the size of non-regular rules. One may consider parallelizing the computation of assumption programs since they are disjoint from each other according to the translations. 


In a sense, our translations are similar to the meta-programming approach to handle preference in ASP (e.g., \cite{delgrande03aframework}) in that we turn LPOD and $\crpt$ into answer set programs that do not have the built-in notion of preference. 

In~\cite{brewka02implementing}, LPOD is implemented using {\sc smodels}. The implementation interleaves the execution of two programs--a generator which produces candidate answer sets and a tester which checks whether a given candidate answer set is maximally preferred or produces a more preferred candidate if it is not. An implementation of $\crp$ reported in~\cite{balduccini07cr-models} uses a similar algorithm. 
In contrast, the reductions shown in this paper can be computed by calling an answer set solver one time without the need for iterating the generator and the tester. This feature may be useful for debugging LPOD and $\crpt$ programs because it allows us to compare all candidate and preferred answer sets globally.

Asprin \cite{brewka15asprin} provides a flexible way to express various preference relations over answer sets and is implemented in {\sc clingo}. Similar to the existing LPOD solvers, {\sc clingo} makes iterative calls to find preferred answer sets, unlike the one-shot execution as we do.

Asuncion {\sl et al.} \citeyear{asuncion14logic} presents a first-order semantics of logic programs with ordered disjunction by translation into second-order logic whereas our translation is into the standard answer set programs.


\BOCC
To the best of our knowledge, there is no implementation for $\crpt$. 
The existing implementations for LPOD \cite{brewka04logic} and $\crp$ \cite{bald07} are all based on a generator (an ASP program that creates all candidate or generalized answer sets respectively) and a tester (an ASP program that checks whether a given answer set is preferred), which require multiple translation and interleaved executing of ASP programs. In contrast, our translation yields an implementation of LPOD, $\crp$
\footnote{
Since $\crp$ is a special case of $\crpt$ where there is no ordered regular rule or ordered cr-rule, our translation is also capable of reducing $\crp$ into ASP program.
}
, and $\crpt$ with one-pass translation and one single execution time. 
However, the size of ${\sf lpod2asp}(\Pi)$ and ${\sf crp2asp}(\Pi)$ after grounding is exponential to the size of the non-regular rules in $\Pi$. This is because our translation will compare all possible assumption programs whose number is exponential to the size of non-regular rules in nature. 
One potential way to improve the efficiency of our translation is to execute each assumption program in parallel, which is doable since each assumption program in our translation is disjoint with each other.
\EOCC

\medskip\noindent
{\bf Acknowledgements:} 
We are grateful to the anonymous referees for their useful comments. This work was partially supported by the National Science Foundation under Grant IIS-1526301.

\bibliographystyle{acmtrans}


\input{lpod-crprolog-tplp-appendix-0501}

\end{document}

%% file: lpod-crprolog-tplp-appendix-0501.tex
\newpage
 \setcounter{page}{1}
\title{Appendix: Translating LPOD and $\crpt$ into Standard Answer Set Programs}

\begin{center}
{\large\textnormal{Online appendix for the paper}}   \\
\medskip
{\Large {Translating LPOD and $\crpt$ into Standard Answer Set Programs}}\\

\medskip
{\large\textnormal{published in Theory and Practice of Logic Programming}}

\medskip
Joohyung Lee and Zhun Yang \\ 
{\sl School of Computing, Informatics and Decision Systems Engineering \\
Arizona State University, Tempe, AZ, USA\\
\email{\{joolee, zyang90\}@asu.edu}}

\end{center}

\thispagestyle{empty}

\begin{appendix}

\section{Proof of Proposition \ref{prop:lpod2asp:hardsplit} } \label{sec:proof:prop:lpod2asp:hardsplit}

Let $S$ be a set of atoms and let $\sigma$ be a signature. By $S|_{\sigma}$, we denote the projection of $S$ onto $\sigma$. Let $S'$ be a set of atoms. We say $S$ {\em agrees with} $S'$ {\em onto} $\sigma$ if $S|_{\sigma} = S'|_{\sigma}$.

In the following proofs, whenever we talk about an LPOD program $\Pi$, we refer to \eqref{program:lpod} as its ordered disjunction part $\Pi_{od}$.

\begin{lemma} \label{lem:1}
Let $\Pi$ be an answer set program, $S$ an answer set of $\Pi$, and $A$ an atom in $S$. 
\begin{itemize}
\ii[{\bf (a)}]
$S$ is an answer set of $\Pi \cup \{ A \leftarrow body \}$. 

\ii[{\bf (b)}]
$S$ is an answer set of $\Pi \cup \{ head \leftarrow body \}$ if $S\not \vDash body$.

\ii[{\bf (c)}]
$S$ is an answer set of $\Pi \setminus \{ head \leftarrow body \}$ if $S\not \vDash body$.

\ii[{\bf (d)}]
$S$ is an answer set of $\Pi \cup \{ constraint \}$ if $S\vDash constraint$.

\ii[{\bf (e)}]
$S$ is an answer set of $\Pi \setminus \{ constraint \}$ if $S\vDash constraint$.
\end{itemize}
Here, $body$ is a conjunction of atoms in $\Pi$ where each atom is possibly preceded by $\no$, $head$ is a disjunction of atoms in $\Pi$, and $constraint$ is a rule of the form $\leftarrow body$.
\end{lemma}

\begin{lemma} \label{lem:1+}
Let $\Pi$ be an answer set program. Let $r$ be a rule of the form $A \leftarrow B_1,\dots,B_m, \no\ C_1, \dots, \no\ C_n$ where $A, B_i, C_j$ are atoms. Let $S$ be a set of atoms such that $S \cap \{ C_1,\dots,C_n \} = \phi$. Then $S$ is an answer set of $\Pi \cup \{r\}$ iff $S$ is an answer set of $\Pi \cup \{ A \leftarrow B_1,\dots,B_m \}$.
\end{lemma}

\begin{lemma} \label{lem:2} (Proposition 8 in \cite{ferraris11logic})
Let $\Pi$ be an ASP program, $Q$ be a set of atoms not occurring in $\Pi$. For each $q \in Q$, let $Def(q)$ be a formula that doesn't contain any atoms from $Q$. Then $X \mapsto X\setminus Q$ is a 1-1 correspondence between the answer sets of $\Pi \cup \{ Def(q) \rightarrow q : q\in Q\}$ and the answer sets of $\Pi$.
\end{lemma}

\medskip
Let $\Pi$ be an LPOD with signature $\sigma$.
By the definition of a split program of LPOD, there are $n_1 \times \dots \times n_m$ split programs of $\Pi$. Let $\Pi(k_1, \dots, k_m)$ denote a split program of $\Pi$, where for $1 \leq i \leq m$, $k_i \in \{1,\dots, n_i \}$ and rule $i$ in $\Pi$ is replaced by its $k_i$-th option:
\beq
C^{k_i}_i \leftarrow \i{Body}_i, \no\ C^1_i, \dots, \no\ C^{k_i-1}_i
\eeq {proof:lpod:0}
where $\i{Body}_i$ is the body of rule $i$.

Let $AP_{\Pi}(x_1, \dots, x_m)$, where $x_i \in \left[ 0, n_i\right]$, denote the assumption program obtained from $\Pi$ by replacing each LPOD rule $i$ with its $x_i$-th assumption, $O_i(x_i)$:
\begin{align}
body_i & ~~\leftarrow~~ \i{Body}_i \label{proof:lpod:body} \\
\bot & ~~\leftarrow~~ x_i = 0,\ body_i \label{proof:lpod:1} \\
\bot & ~~\leftarrow~~ x_i > 0,\ \no\ body_i \label{proof:lpod:2}\\
C^{j}_i & ~~\leftarrow~~ body_i,\ x_i=j & \text{(for $1\leq j \leq n_i$)} \label{proof:lpod:3} \\
\bot & ~~\leftarrow~~ body_i,\ x_i\ne j,\ \no\ C^1_i, \dots, \no\ C^{j-1}_i,\  C^j_i & \text{(for $1\leq j \leq n_i$)} \label{proof:lpod:4}
\end{align}
where $\i{Body}_i$ is the body of rule $i$, and $body_i$ is an atom not occurring in $\Pi$.

\noindent{\bf Proposition~\ref{prop:lpod2asp:hardsplit} \optional{prop:lpod2asp:hardsplit}}\ \ 
{\sl
For any LPOD $\Pi$ of signature $\sigma$ and any set $S$ of atoms of $\sigma$,
$S$ is a candidate answer set of $\Pi$ iff $S\cup \{ body_i \mid$ $S$ satisfies the body of rule $i$ in $\Pi_{od} \}$ is an answer set of some assumption program of $\Pi$. 
More specifically,
\begin{itemize}
\ii[{\bf (a)}]
for any candidate answer set $S$ of $\Pi$, let's obtain $x_1, \dots, x_m$ such that, for $1\leq i \leq m$,
\begin{itemize}
\item
$x_i = 0$ if $S\not \vDash \i{Body}_i$,
\item
$x_i=k$ if $S \vDash \i{Body}_i$, and $C^{k}_i \in S$, and $C^j_i \not \in S$ for $1\leq j \leq k-1$,
\end{itemize}
then $\phi(S) = S\cup \{ body_i \mid$ $S$ satisfies the body of rule $i$ in $\Pi_{od} \}$ is an answer set of $AP_{\Pi}(x_1, \dots, x_m)$;
\ii[{\bf (b)}]
for any answer set $S'$ of any assumption program $AP_{\Pi}(x_1, \dots, x_m)$, $S'|_{\sigma}$ is a candidate answer set of $\Pi$.
\end{itemize}
}
\medskip

\begin{proof}
\begin{itemize}
\ii[{\bf (a)}]
Let $S$ be a candidate answer set of $\Pi$. We obtain $x_1, \dots, x_m$ such that, for $1\leq i \leq m$,
\begin{itemize}
\item
$x_i = 0$ if $S\not \vDash \i{Body}_i$,
\item
$x_i=k$ if $S \vDash \i{Body}_i$, and $C^{k}_i \in S$, and $C^j_i \not \in S$ for $1\leq j \leq k-1$.
\end{itemize}
We will prove that $\phi(S)$ is an answer set of $AP_{\Pi}(x_1, \dots, x_m)$. 
Since $S$ is a candidate answer set of $\Pi$, $S$ must be an answer set of some $\Pi(k_1, \dots, k_m)$. Let's consider any LPOD rule $i$ in $\Pi$. We know rule $i$ is replaced by one of its options (\ref{proof:lpod:0}) in $\Pi(k_1, \dots, k_m)$. Let's obtain $\Pi'$ from $\Pi(k_1, \dots, k_m)$ by replacing the option of rule $i$ with $O_i(x_i)$. Recall that $\i{Body}_i$ represent the body of rule $i$. Let $S'$ be $S \cup \{ body_i \mid$ $S\vDash \i{Body}_i$ $\}$. We are going to prove $S'$ is an answer set of $\Pi'$.

Since $x_i=j$ is not an atom, rule (\ref{proof:lpod:4}) is strong equivalent to the following constraint 
\[
\leftarrow body_i, C^j_i,  \no\ C^1_i, \dots, \no\ C^{j-1}_i, \no\ x_i=j 
\]
thus Lemma \ref{lem:1} (d) applies to this rule.
According to the assignments for $x_1, \dots, x_m$, it's obvious that rules \eqref{proof:lpod:1}, \eqref{proof:lpod:2}, \eqref{proof:lpod:4} are satisfied by $\phi(S)$.

\begin{itemize}
\item
If $S \not \vDash \i{Body}_i$, $S' \not \vDash body_i$. By Lemma \ref{lem:1} (c), $S$ is an answer set of $\Pi(k_1, \dots, k_m)$ minus the option of rule $i$. Since rules \eqref{proof:lpod:1}, \eqref{proof:lpod:2}, \eqref{proof:lpod:4} are satisfied by $S$, and the bodies of rules \eqref{proof:lpod:body}, \eqref{proof:lpod:3} are not satisfied by $S$, by Lemma \ref{lem:1} (d) and Lemma \ref{lem:1} (b), $S' = S$ is an answer set of $\Pi'$.

\item
If $S \vDash \i{Body}_i$, then $S' \vDash body_i$, and $x_i>0$, and at least one of the atoms in $\{C^1_i, \dots, C^{n_i}_i\}$ must be true, and the first atom among them that is true in $S$ is $C^{x_i}_i$ ($S$ satisfies $C^{x_i}_i$ and $S$ doesn't satisfy $C^j_i$ for $j\in \{ 1,\dots,x_i-1 \}$). 
Let $\Pi''$ be the union of $\Pi(k_1, \dots, k_m)$ and the rule \eqref{proof:lpod:body}, then by Lemma \ref{lem:2}, $S'$ is an answer set of $\Pi''$. 
Assume for the sake of contradiction that $k_i < x_i$. By rule (\ref{proof:lpod:0}), at least one of $\{C^1_i, \dots, C^{k_i}_i\}$ must be true in $S$, which contradicts with the fact that the first atom that is true in $S$ is $C^{x_i}_i$. 
\footnote{
For example, suppose $k_i=2$, and $x_i=3$ is the index of the first atom in $\{C^1_i, \dots, C^{n_1}_i\}$ that is true in $S$. Since $S$ satisfies the $k_i$-th option of rule $i$ --- ``$C^2 \leftarrow body, \no\ C^1$'', and $S\vDash body$, then either $C^1$ is true or $C^2$ is true, which contradicts with the fact that $C^{3}$ is the first atom to be true in $S$.
}
Then there are 2 cases for $k_i$: 
\begin{itemize}
\item
if $k_i = x_i$, by Lemma~\ref{lem:1+}, $S'$ is an answer set of $\Pi'' \cup \{C^{x_i}_i \leftarrow body_i \}$ minus rule  \eqref{proof:lpod:0}. Consequently, by Lemma~\ref{lem:1} (b), $S'$ is an answer set of $\Pi''$ union rule \eqref{proof:lpod:3} minus rule \eqref{proof:lpod:0}. Since rules \eqref{proof:lpod:1}, \eqref{proof:lpod:2}, \eqref{proof:lpod:4} are satisfied by $S'$, by Lemma \ref{lem:1} (d), $S'$ is an answer set of $\Pi'$; 
\item
if $k_i > x_i$, ``$\no\ C^{x_i}_i$'' is in the body of rule \eqref{proof:lpod:0}, then by Lemma \ref{lem:1} (c), $S'$ is an answer set of $\Pi''$ minus rule \eqref{proof:lpod:0}. Since $S\vDash C^{x_i}_i$, by Lemma \ref{lem:1} (a), $S'$ is an answer set of $\Pi'' \cup \{C^{x_i}_i \leftarrow body_i \}$ minus rule  \eqref{proof:lpod:0}. Consequently, by Lemma~\ref{lem:1} (b), $S'$ is an answer set of $\Pi''$ union rule \eqref{proof:lpod:3} minus rule \eqref{proof:lpod:0}. Since rules \eqref{proof:lpod:1}, \eqref{proof:lpod:2}, \eqref{proof:lpod:4} are satisfied by $S'$, by Lemma \ref{lem:1} (d), $S'$ is an answer set of $\Pi'$.
\end{itemize}
\end{itemize}

Consequently, $\phi(S)$ is an answer set of $AP_{\Pi}(x_1, \dots, x_m)$, which is obtained from $\Pi(k_1, \dots, k_m)$ by replacing each option of rule $i$ of $\Pi$ with $O_i(x_i)$ for $1\leq i \leq m$. 

\ii[{\bf (b)}]
Let $S'$ be an answer set of program $AP_{\Pi}(x_1, \dots, x_m)$. 
Let's consider any LPOD rule $i$ in $\Pi$. Let's obtain $\Pi'$ from $AP_{\Pi}(x_1, \dots, x_m)$ by replacing $O_i(x_i)$ with the $k_i$-th option of rule $i$ where $k_i = x_i$ if $x_i>0$, $k_i=1$ if $x_i=0$.
We first prove $S = S' \setminus \{ body_i\}$ is an answer set of $\Pi'$.

Since $S'$ must satisfy rules \eqref{proof:lpod:1}, \eqref{proof:lpod:2}, \eqref{proof:lpod:4}, by Lemma \ref{lem:1} (e), $S'$ is an answer set of $AP_{\Pi}(x_1, \dots, x_m)$ minus rules \eqref{proof:lpod:1}, \eqref{proof:lpod:2}, \eqref{proof:lpod:4}. By Lemma~\ref{lem:1} (c), $S'$ is an answer set of $AP_{\Pi}(x_1, \dots, x_m) \cup \{ C^{x_i}_i \leftarrow body_i \}$ minus rules \eqref{proof:lpod:1}, \eqref{proof:lpod:2}, \eqref{proof:lpod:3}, \eqref{proof:lpod:4}. 
Note that by rule \eqref{proof:lpod:body}, $S'$ satisfies $body_i$ iff $S'$ satisfies $\i{Body}_i$. There are 2 cases as follows.
\begin{itemize}
\item
If $S' \vDash \i{Body}_i$, $S' \vDash body_i$. Since $S'$ satisfies rules \eqref{proof:lpod:1} and \eqref{proof:lpod:3}, we know $x_i>0$ and $S'$ satisfies $C^{x_i}_i$. Thus $k_i$ equals to $x_i$.
Assume for the sake of contradiction that the first atom among $\{ C^1_i,\dots,C^{n_i}_i \}$ that is true in $S'$ is $C^j_i$ and $j<x_i$. Since $S'$ satisfies rule \eqref{proof:lpod:4}, $S'$ satisfies $x_i = j$. Contradiction. Thus $S'$ satisfies $C^{x_i}_i$ and doesn't satisfy $C^j_i$ for $j\in \{ 1,\dots,x_i-1 \}$. By Lemma~\ref{lem:1+}, $S'$ is an answer set of $AP_{\Pi}(x_1, \dots, x_m)$ union rule \eqref{proof:lpod:0} minus rules \eqref{proof:lpod:1}, \eqref{proof:lpod:2}, \eqref{proof:lpod:3}, \eqref{proof:lpod:4}. By Lemma~\ref{lem:2}, $S$ is an answer set of $\Pi'$.

\item
If $S' \not \vDash \i{Body}_i$, $S' \not \vDash body_i$. By lemma~\ref{lem:1} (c), $S'$ is an answer set of $AP_{\Pi}(x_1, \dots, x_m)$ minus rules \eqref{proof:lpod:body}, \eqref{proof:lpod:1}, \eqref{proof:lpod:2}, \eqref{proof:lpod:3}, \eqref{proof:lpod:4}. 
By Lemma~\ref{lem:1} (b), $S = S'$ is an answer set of $\Pi'$.
\end{itemize}

So $S$ is an answer set of $\Pi'$. Consequently, $S'|_{\sigma}$ is an answer set of $\Pi(k_1,\dots,k_m)$, where $k_i = x_i$ if $x_i>0$, $k_i=1$ if $x_i=0$. In other words, $S'|_{\sigma}$ is a candidate answer set of $\Pi$.

\end{itemize}
\end{proof}

\section{Proof of Proposition \ref{prop:lpod2asp:base} } \label{sec:proof:prop:lpod2asp:base}
For any answer set program $\Pi$, let $gr(\Pi, x_1,\dots, x_m)$ be a partial grounded program obtained from $\Pi$ by replacing variables $X_1,\dots, X_m$ in $\Pi$ with $x_1,\dots, x_m$.

Let $\Pi$ be an LPOD of signature $\sigma$. In the following proofs, let ${\sf lpod2asp}(\Pi)$ be $\Pi_1 \cup \Pi_2 \cup \Pi_3$, where $\Pi_1$ consists of the rules in bullets 1 and 2 in section {\bf Generate Candidate Answer Sets}, $\Pi_2$ consists of the rules in bullet 3 in the same section, and $\Pi_3$ consists of the rules in section {\bf Find Preferred Answer Sets}. Note that ${\sf lpod2asp}(\Pi)_{base}$ is $\Pi_1 \cup \Pi_2$.

\medskip
The proof of {\bf Proposition \ref{prop:lpod2asp:base}} will use a restricted version of the splitting theorem from \cite{ferr09b}, which is reformulated as follows:

\noindent{\bf Splitting Theorem}\ \ 
{\sl
Let $\Pi_1$, $\Pi_2$ be two answer set programs,
${\bf p}$, ${\bf q}$ be disjoint tuples of distinct atoms. If
\begin{itemize}
\item each strongly connected component of the dependency graph of $\Pi_1\cup\Pi_2$ w.r.t. ${\bf p}\cup{\bf q}$ is a subset of ${\bf p}$ or a subset of ${\bf q}$,
\item no atom in ${\bf p}$ has a strictly positive occurrence in $\Pi_2$, and
\item no atom in ${\bf q}$ has a strictly positive occurrence in $\Pi_1$,
\end{itemize}
then an interpretation $I$ of $\Pi_1\cup \Pi_2$ is an answer set of $\Pi_1\cup \Pi_2$ relative to ${\bf p}\cup{\bf q}$ if and only if $I$ is an answer set of $\Pi_1$ relative to ${\bf p}$ and $I$ is an answer set of $\Pi_2$ relative to ${\bf q}$.
}

\bigskip

\noindent{\bf Proposition~\ref{prop:lpod2asp:base} \optional{prop:lpod2asp:base}}\ \ 
{\sl
The candidate answer sets of an LPOD $\Pi$ of signature $\sigma$ are exactly the candidate answer sets on $\sigma$ of ${\sf lpod2asp}(\Pi)_{base}$. In other words,
(for any set $S$, let $\phi(S)$ be $S \cup \{ body_i \mid$ $S$ satisfies the body of rule $i$ in $\Pi_{od}$ $\}$)
\begin{itemize}
\ii[{\bf (a)}]
for any candidate answer set $S$ of $\Pi$, there are $x_1,\dots,x_m$ such that $\phi(S)$ is an answer set of $AP_{\Pi}(x_1, \dots, x_m)$, and there exists an optimal answer set $S'$ of $\Pi_1 \cup \Pi_2$ such that $S' \vDash ap(x_1, \dots, x_m)$ and $S = shrink(S', x_1, \dots, x_m)$;
\ii[{\bf (b)}]
for any optimal answer set $S'$ of $\Pi_1 \cup \Pi_2$ and any $x_1,\dots,x_m$ such that $S' \vDash ap(x_1,\dots,x_m)$, $S = shrink(S', x_1, \dots, x_m)$ is a candidate answer set of $\Pi$, and $\phi(S)$ is an answer set of $AP_{\Pi}(x_1, \dots, x_m)$.
\end{itemize}
}

\begin{proof}
Let $\Pi_{1,2}$ be $\Pi_1 \cup \Pi_2$. According to the translation, the empty set is always an answer set of $\Pi_{1,2}$ (since the empty set doesn't satisfy the body of any rule in $\Pi_{1,2}$), thus there must exist at least one optimal answer set of $\Pi_{1,2}$. Furthermore, by rule (\ref{equ:lpod:wc}), the optimal answer set should contain as many $ap(*)$ as possible. 
Then $gr(\Pi_{1,2}, x_1,\dots, x_m)$ is $gr(\Pi_{1}, x_1,\dots, x_m) \cup gr(\Pi_{2}, x_1,\dots, x_m)$.
Let $\Pi_{1,2}^{gr}$ be $\bigcup_{y_i \in \{ 0,\dots,n_i \}} gr(\Pi_{1,2}, y_1,\dots, y_m)$.
Let $\sigma^{\Pi_{1,2}^{gr}}$ be the signature of $\Pi_{1,2}^{gr}$, let $\sigma^{gr(\Pi_{1,2}, x_1,\dots, x_m)}$ be the signature of $gr(\Pi_{1,2}, x_1,\dots, x_m)$, let $\sigma^{gr(\Pi_{1}, x_1,\dots, x_m)}$ be the signature of $gr(\Pi_{1}, x_1,\dots, x_m)$, and let $\sigma^{gr(\Pi_{2}, x_1,\dots, x_m)}$ be $\sigma^{gr(\Pi_{1,2}, x_1,\dots, x_m)} \setminus \sigma^{gr(\Pi_{1}, x_1,\dots, x_m)}$.
We then prove bullets {\bf (a)} and {\bf (b)} as follows.
\begin{itemize}
\ii[{\bf (a)}]
Let $S$ be a candidate answer set of $\Pi$. 
By Proposition \ref{prop:lpod2asp:hardsplit}, $\phi(S)$ must be an answer set of some $AP_{\Pi}(x_1, \dots, x_m)$ of $\Pi$. Let $\psi(S)$ be 
$$
\ba l
\{ a({\bf v}, x_1, \dots, x_m) \mid a({\bf v}) \in S \}
\cup \{ body_i(x_1, \dots, x_m) \mid S \text{ satisfies the body of rule $i$ in $\Pi_{od}$}\} \\
\cup \{ ap(x_1, \dots, x_m), degree(ap(x_1, \dots, x_m), d_1, \dots, d_m) \},
\ea
$$ where $d_i = 1$ if $x_i = 0$, $d_i = x_i$ if $x_i > 0$. 
Our target is to construct an $S'$ from $\psi(S)$ and prove $S'$ is an optimal answer set of $\Pi_{1,2}$ such that $S' \vDash ap(x_1, \dots, x_m)$, and $S = shrink(S', x_1, \dots, x_m)$.

First, we prove $\psi(S)$ is an optimal answer set of $gr(\Pi_{1,2}, x_1,\dots, x_m)$. 
\begin{enumerate}
\item
By the construction of $\psi(S)$, $\psi(S)$ satisfies the reduct of $gr(\Pi_{2}, x_1,\dots, x_m)$ relative to $\psi(S)$, and is minimal with respect to $\sigma^{gr(\Pi_{2}, x_1,\dots, x_m)}$. So $\psi(S)$ is an answer set of $gr(\Pi_{2}, x_1,\dots, x_m)$ with respect to $\sigma^{gr(\Pi_{2}, x_1,\dots, x_m)}$.
\item
Since $\phi(S)$ is a minimal model of the reduct of $AP_{\Pi}(x_1, \dots, x_m)$ relative to $\phi(S)$, and $\psi(S) \vDash ap(x_1, \dots, x_m)$, it's easy to check that $\psi(S)$ is a minimal model of the reduct of $gr(\Pi_{1}, x_1,\dots, x_m)$ relative to $\psi(S)$ with respect to $\sigma^{gr(\Pi_{1}, x_1,\dots, x_m)}$. So $\psi(S)$ is an answer set of $gr(\Pi_{1}, x_1,\dots, x_m)$ with respect to $\sigma^{gr(\Pi_{1}, x_1,\dots, x_m)}$. 
\end{enumerate}
By the splitting theorem, $\psi(S)$ is an answer set of $gr(\Pi_{1,2}, x_1,\dots, x_m)$. 
Since $\psi(S)$ satisfies $ap(x_1, \dots, x_m)$, which is the only $ap(*)$ occurring in $gr(\Pi_{1,2}, x_1,\dots, x_m)$, $\psi(S)$ must be an optimal answer set of $gr(\Pi_{1,2}, x_1,\dots, x_m)$. 


Then, we construct an optimal answer set $S'$ of $\Pi_{1,2}$ from any optimal answer set $S''$ of $\Pi_{1,2}$ such that $S' \vDash ap(x_1, \dots, x_m)$ and $S = shrink(S', x_1, \dots, x_m)$. 

We first show that $S''$ must satisfy $ap(x_1,\dots,x_m)$. 
Assume for the sake of contradiction that $S''$ does not satisfy $ap(x_1,\dots,x_m)$. 
Since each partial grounded program of $\Pi_{1,2}$ is disjoint from each other, by the splitting theorem, $S''|_{\sigma^{gr(\Pi_{1,2}, x_1,\dots, x_m)}}$ is an answer set of $gr(\Pi_{1,2}, x_1,\dots, x_m)$ and $S'' \setminus S''|_{\sigma^{gr(\Pi_{1,2}, x_1,\dots, x_m)}}$ is an answer set of $\Pi_{1,2}^{gr} \setminus gr(\Pi_{1,2}, x_1,\dots, x_m)$. 
Let $S'$ be the union of $\psi(S)$ and $S'' \setminus S''|_{\sigma^{gr(\Pi_{1,2}, x_1,\dots, x_m)}}$, since $\psi(S)$ is an answer set of $gr(\Pi_{1,2}, x_1,\dots, x_m)$, by the splitting theorem, $S'$ is an answer set of $\Pi_{1,2}$. Since $S'$ has a lower penalty than $S''$, $S''$ is not an optimal answer set of $\Pi_{1,2}$, which contradicts with our initial assumption. So $S''$ must satisfy $ap(x_1,\dots,x_m)$. Indeed, if there exists an answer set of $AP_{\Pi}(x_1,\dots,x_m)$,
\beq
\text{
any optimal answer set of $\Pi_{1,2}$ must satisfy $ap(x_1,\dots,x_m)$.
}
\eeq {statement:1}
Consequently, $S'$ has the same penalty as $S''$ in $\Pi_{1,2}$, which means that $S'$ is an optimal answer set of $\Pi_{1,2}$. Besides, $S$ equals to $shrink(\psi(S), x_1,\dots,x_m)$, which equals to $shrink(S', x_1, \dots, x_m)$.

\ii[{\bf (b)}]
Let $S'$ be an optimal answer set of $\Pi_{1,2}$ and $x_1,\dots,x_m$ a list of integers such that 
$S' \vDash ap(x_1,\dots,x_m)$. Our target is to prove $S = shrink(S', x_1, \dots, x_m)$ is a candidate answer set of $\Pi$. By Proposition \ref{prop:lpod2asp:hardsplit}, it is sufficient to prove that $\phi(S)$ is an answer set of $AP_{\Pi}(x_1,\dots,x_m)$.


We first split $\Pi_{1,2}^{gr}$ into $gr(\Pi_{1}, x_1,\dots, x_m)$ and the remaining part $\Pi_{1,2}^{gr} \setminus gr(\Pi_{1}, x_1,\dots, x_m)$.
Since
\begin{enumerate}
\item
no atom in $\sigma^{gr(\Pi_{1}, x_1,\dots, x_m)}$ has a strictly positive occurrence in $\Pi_{1,2}^{gr} \setminus gr(\Pi_{1}, x_1,\dots, x_m)$, 
\item
no atom in $\sigma^{\Pi_{1,2}^{gr}} \setminus \sigma^{gr(\Pi_{1}, x_1,\dots, x_m)}$ has a strictly positive occurrence in $gr(\Pi_{1}, x_1,\dots, x_m)$, and 
\item
each strongly connected component of the dependency graph of $\Pi_{1,2}$ w.r.t. $\sigma^{\Pi_{1,2}^{gr}}$ is a subset of $\sigma^{gr(\Pi_{1}, x_1,\dots, x_m)}$ or $\sigma^{\Pi_{1,2}^{gr}} \setminus \sigma^{gr(\Pi_{1}, x_1,\dots, x_m)}$, 
\end{enumerate}
by the splitting theorem, $S'$ is an answer set of $gr(\Pi_{1}, x_1,\dots, x_m)$ with respect to $\sigma^{gr(\Pi_{1}, x_1,\dots, x_m)}$.
So $S'|_{\sigma^{gr(\Pi_{1}, x_1,\dots, x_m)}}$ is a minimal model of the reduct of $gr(\Pi_{1}, x_1,\dots, x_m)$ relative to $S'|_{\sigma^{gr(\Pi_{1}, x_1,\dots, x_m)}}$. 
Since $S'|_{\sigma^{gr(\Pi_{1}, x_1,\dots, x_m)}} \vDash ap(x_1,\dots,x_m)$, it's easy to check that $\phi(S)$ is a minimal model of the reduct of $AP_{\Pi}(x_1, \dots, x_m)$ relative to $\phi(S)$, where the reduct can be obtained from the reduct of $gr(\Pi_{1}, x_1,\dots, x_m)$ relative to $S'|_{\sigma^{gr(\Pi_{1}, x_1,\dots, x_m)}}$ by replacing each occurrence of $ap(x_1, \dots, x_m)$ with $\top$, and replacing each occurrence of $a({\bf v}, x_1,\dots,x_m)$ by $a({\bf v})$ where $a({\bf v}) \in \sigma$.
Thus $\phi(S)$ is an answer set of $AP_{\Pi}(x_1,\dots,x_m)$. By Proposition \ref{prop:lpod2asp:hardsplit}, $S$ is a candidate answer set of $\Pi$.

\end{itemize}
\end{proof}

\section{Proof of Theorem \ref{thm:lpod2asp} } \label{sec:proof:thm:lpod2asp}

Let $\Pi$ be an LPOD of signature $\sigma$. Recall that we let ${\sf lpod2asp}(\Pi)$ be $\Pi_1 \cup \Pi_2 \cup \Pi_3$, where $\Pi_1$ consists of the rules in bullets 1 and 2 in section {\bf Generate Candidate Answer Sets}, $\Pi_2$ consists of the rules in bullet 3 in the same section, and $\Pi_3$ consists of the rules in section {\bf Find Preferred Answer Sets}. 

\begin{lemma} \label{prop:lpod2asp:base-3}
Let $\Pi$ be an LPOD. 
There is a 1-1 correspondence between the answer sets of ${\sf lpod2asp}(\Pi)$ and the answer sets of $\Pi_1 \cup \Pi_2$, and any answer set of ${\sf lpod2asp}(\Pi)$ agrees with the corresponding answer set of $\Pi_1 \cup \Pi_2$ on the signature of $\Pi_1 \cup \Pi_2$.
\end{lemma}

\begin{proof}
Let $\Pi_{1,2}$ be $\Pi_1 \cup \Pi_2$. 
Let's take $\Pi_{1,2}$ as our current program, $\Pi_{cur}$, and consider including the translation rules in $\Pi_3$ (rules \eqref{equ:lpod:card:card} --- \eqref{equ:lpod:pena:prf} under each preference criterion) into $\Pi_{cur}$. 
For each criterion, let's include the first rule, e.g., rule (\ref{equ:lpod:card:card}), into $\Pi_{cur}$, it's easy to see that this rule satisfies the condition of Lemma \ref{lem:2}. By Lemma \ref{lem:2}, there is a 1-1 correspondence between the answer sets of $\Pi_{cur}$ and the answer sets of $\Pi_{1,2}$. Similarly, if we further include the second rule, e.g., rule (\ref{equ:lpod:card:equ2}), into $\Pi_{cur}$, there is still a 1-1 correspondence between the answer sets of $\Pi_{cur}$ and the answer sets of $\Pi_{1,2}$. Similarly, we can include more rules from $\Pi_3$ into the current program $\Pi_{cur}$ in order, and consequently, there is a 1-1 correspondence between the answer sets of $\Pi_{1,2} \cup \Pi_3$ and the answer sets of $\Pi_{1,2}$. Since all the atoms introduced by $\Pi_3$ are not in the signature of $\Pi_{1,2}$, any answer set of ${\sf lpod2asp}(\Pi)$ agrees with the corresponding answer set of $\Pi_{1,2}$ on the signature of $\Pi_{1,2}$.
\end{proof}

\begin{lemma} \label{prop:lpod2asp:hardsplit:degree}
Let $S$ be a candidate answer set of an LPOD $\Pi$. 
If $\phi(S) = S \cup \{ body_i \mid$ $S$ satisfies the body of rule $i$ in $\Pi_{od}\}$ is an answer set of $AP_{\Pi}(x_1,\dots,x_m)$ for some $x_1,\dots,x_m$, then for $1\leq i \leq m$, $S$ satisfies rule $i$ of $\Pi_{od}$ to degree 1 if $x_i=0$, to degree $x_i$ if $x_i>0$.
\footnote{
This lemma won't hold if $AP_{\Pi}(x_1, \dots, x_m)$ is replaced by $\Pi(k_1, \dots, k_m)$.
}
\end{lemma}

\begin{proof}
Since $\phi(S)$ is an answer set of $AP_{\Pi}(x_1, \dots, x_m)$, for $1\leq i \leq m$, $S$ satisfies rules \eqref{proof:lpod:1}, \eqref{proof:lpod:2}, which are equivalent to:
\[
\ba {l}
x_i = 0 \leftrightarrow \neg body_i \\
x_i > 0 \leftrightarrow body_i
\ea
\]
If $x_i=0$, $\phi(S) \not \vDash body_i$. So the body of rule $i$ is not satisfied by $S$, which means rule $i$ is satisfied at (i.e., satisfied to) degree 1.
If $x_i > 0$, $\phi(S) \vDash body_i$. By rule (\ref{proof:lpod:4}), the first atom in the head of rule $i$ that is true in $\phi(S)$, and also $S$, is $C^{x_i}$, which means that rule $i$ is satisfied by $S$ at degree $x_i$.
\end{proof}

\begin{lemma} \label{lem:lpod:twoAS}
Let $\Pi$ be an LPOD (\ref{program:lpod}). Let $AP_{\Pi}(x_1,\dots,x_m)$ and $AP_{\Pi}(y_1, \dots, y_m)$ be two programs that are consistent, where the list $x_1,\dots,x_m$ is different from $y_1, \dots, y_m$. Let $S_1$ be an answer set of $AP_{\Pi}(x_1,\dots,x_m)$, $S_2$ be an answer set of $AP_{\Pi}(y_1, \dots, y_m)$.
Then 
\begin{itemize}
\ii[{\bf (a)}]
there exists an optimal answer set $K$ of ${\sf lpod2asp}(\Pi)$ such that $K \vDash ap(x_1, \dots, x_m)$,  $K \vDash ap(y_1, \dots, y_m)$, $S_1|_{\sigma} = shrink(K, x_1, \dots, x_m)$, and $S_2|_{\sigma} = shrink(K, y_1, \dots, y_m)$;
\ii[{\bf (b)}]
any optimal answer set $K$ of ${\sf lpod2asp}(\Pi)$ must satisfy $ap(x_1, \dots, x_m)$ and $ap(y_1, \dots, y_m)$.
\end{itemize}
\end{lemma}

\begin{proof}
{\bf (a)} Let ${\sf lpod2asp}(\Pi)$ be $\Pi_1 \cup \Pi_2 \cup \Pi_3$ as defined before. By Lemma \ref{prop:lpod2asp:base-3}, it is sufficient to prove that there exists an optimal answer set $L$ of $\Pi_1 \cup \Pi_2$ such that $L \vDash ap(x_1, \dots, x_m)$,  $L \vDash ap(y_1, \dots, y_m)$, $S_1|_{\sigma} = shrink(L, x_1, \dots, x_m)$, and $S_2|_{\sigma} = shrink(L, y_1, \dots, y_m)$.

Let $\Pi_{1,2}$ be $\Pi_1 \cup \Pi_2$. By Proposition \ref{prop:lpod2asp:base}, there exists an optimal answer set $L_2$ of $\Pi_{1,2}$ such that $L_2 \vDash ap(y_1, \dots, y_m)$, and $S_2|_{\sigma} = shrink(L_2, y_1, \dots, y_m)$. 
Let $\psi(S_1)$ be 
$
\{ a({\bf v}, x_1, \dots, x_m) \mid a({\bf v}) \in S_1 \}
~~\cup~~ \{ body_i(x_1, \dots, x_m)~ \mid ~S_1 \text{ satisfies the body of rule $i$ in $\Pi_{od}$}\}
~~\cup~~ \{ ap(x_1, \dots, x_m),$  
\newline
$degree(ap(x_1, \dots, x_m), d_1, \dots, d_m) \} 
$, where $d_i = 1$ if $x_i = 0$, $d_i = x_i$ if $x_i > 0$. 
Let $L$ be the union of $\psi(S_1)$ and $L_2 \setminus L_2|_{\sigma^{gr(\Pi_{1,2}, x_1,\dots, x_m)}}$.
It's easy to see that $L \vDash ap(x_1, \dots, x_m)$,  $L \vDash ap(y_1, \dots, y_m)$, $S_1|_{\sigma} = shrink(L, x_1, \dots, x_m)$, and $S_2|_{\sigma} = shrink(L, y_1, \dots, y_m)$. Besides, $L$ has the same penalty as $L_2$.
So to prove Lemma \ref{lem:lpod:twoAS} {\bf (a)}, it is sufficient to prove that $L$ is an answer set of $\Pi_1 \cup \Pi_2$.

First, we prove $\psi(S_1)$ is an answer set of $gr(\Pi_{1,2}, x_1,\dots, x_m)$. 
\begin{enumerate}
\item
By the construction of $\psi(S_1)$, $\psi(S_1)$ satisfies the reduct of $gr(\Pi_{2}, x_1,\dots, x_m)$ relative to $\psi(S_1)$, and is minimal with respect to $\sigma^{gr(\Pi_{2}, x_1,\dots, x_m)}$. So $\psi(S_1)$ is an answer set of $gr(\Pi_{2}, x_1,\dots, x_m)$ relative to $\sigma^{gr(\Pi_{2}, x_1,\dots, x_m)}$.
\item
Since $S_1$ is a minimal model of the reduct of $AP_{\Pi}(x_1, \dots, x_m)$ relative to $S_1$, and $\psi(S_1) \vDash ap(x_1, \dots, x_m)$, it's easy to check that $\psi(S_1)$ is a minimal model of the reduct of $gr(\Pi_{1}, x_1,\dots, x_m)$ relative to $\psi(S_1)$ with respect to $\sigma^{gr(\Pi_{1}, x_1,\dots, x_m)}$. So $\psi(S_1)$ is an answer set of $gr(\Pi_{1}, x_1,\dots, x_m)$ relative to $\sigma^{gr(\Pi_{1}, x_1,\dots, x_m)}$. 
\end{enumerate}
By the splitting theorem, $\psi(S_1)$ is an answer set of $gr(\Pi_{1,2}, x_1,\dots, x_m)$. 

Second, let $\Pi_{1,2}^{gr}$ be $\bigcup_{y_i \in \{ 0,\dots,n_i \}} gr(\Pi_{1,2}, y_1,\dots, y_m)$. Since each partial grounded program of $\Pi_{1,2}$ is disjoint from each other, by the splitting theorem, $L_2|_{\sigma^{gr(\Pi_{1,2}, x_1,\dots, x_m)}}$ is an answer set of $gr(\Pi_{1,2}, x_1,\dots, x_m)$ and $L_2 \setminus L_2|_{\sigma^{gr(\Pi_{1,2}, x_1,\dots, x_m)}}$ is an answer set of $\Pi_{1,2}^{gr} \setminus gr(\Pi_{1,2}, x_1,\dots, x_m)$. 

Finally, by the splitting theorem, $L$ is an answer set of $\Pi_1 \cup \Pi_2$.

{\bf (b)} Let ${\sf lpod2asp}(\Pi)$ be $\Pi_1 \cup \Pi_2 \cup \Pi_3$ as defined before. By Lemma \ref{prop:lpod2asp:base-3}, it is sufficient to prove that any optimal answer set $L$ of $\Pi_1 \cup \Pi_2$ must satisfy $ap(x_1, \dots, x_m)$ and $ap(y_1, \dots, y_m)$. 
Since $S_1$ is an answer set of $AP_{\Pi}(x_1,\dots,x_m)$, and $S_2$ is an answer set of $AP_{\Pi}(y_1, \dots, y_m)$, 
by \eqref{statement:1}, any optimal answer set $L$ of $\Pi_1 \cup \Pi_2$ must satisfy $ap(x_1, \dots, x_m)$ and $ap(y_1, \dots, y_m)$.
\end{proof}

\begin{lemma} \label{lem:lpod:candidate}
The candidate answer sets of an LPOD $\Pi$ of signature $\sigma$ are exactly the candidate answer sets on $\sigma$ of ${\sf lpod2asp}(\Pi)$. In other words, (for any set $S$ of atoms, let $\phi(S)$ be $S \cup \{ body_i \mid$ $S$ satisfies the body of rule $i$ in $\Pi_{od}$ $\}$)
\begin{itemize}
\ii[{\bf (a)}]
for any candidate answer set $S$ of $\Pi$, there are $x_1,\dots,x_m$ such that $\phi(S)$ is an answer set of $AP_{\Pi}(x_1, \dots, x_m)$, and there exists an optimal answer set $K$ of ${\sf lpod2asp}(\Pi)$ such that $K \vDash ap(x_1, \dots, x_m)$ and $S = shrink(K, x_1, \dots, x_m)$;
\ii[{\bf (b)}]
for any optimal answer set $K$ of ${\sf lpod2asp}(\Pi)$ and any $x_1,\dots,x_m$ such that $K \vDash ap(x_1,\dots,x_m)$, $S = shrink(K, x_1, \dots, x_m)$ is a candidate answer set of $\Pi$, and $\phi(S)$ is an answer set of $AP_{\Pi}(x_1, \dots, x_m)$.
\end{itemize}
\end{lemma}

\begin{proof}
\begin{itemize}
\ii[{\bf (a)}]
Let $S$ be a candidate answer set of $\Pi$. Let ${\sf lpod2asp}(\Pi)$ be $\Pi_1 \cup \Pi_2 \cup \Pi_3$ as defined before. By Proposition \ref{prop:lpod2asp:base}, there are $x_1,\dots,x_m$ such that $\phi(S)$ is an answer set of $AP_{\Pi}(x_1, \dots, x_m)$, and there exists an optimal answer set $S'$ of $\Pi_1 \cup \Pi_2$ such that $S' \vDash ap(x_1, \dots, x_m)$ and $S = shrink(S', x_1, \dots, x_m)$. By Lemma \ref{prop:lpod2asp:base-3}, there exists an answer set $K$ of ${\sf lpod2asp}(\Pi)$ such that $K$ agrees with $S'$ on the signature of $\Pi_1 \cup \Pi_2$. Thus $K \vDash ap(x_1, \dots, x_m)$ and $S = shrink(K, x_1, \dots, x_m)$. Since the signature of $\Pi_1 \cup \Pi_2$ includes all $ap(*)$ atoms and $S'$ is an optimal answer set of $\Pi_1 \cup \Pi_2$, $K$ is an optimal answer set of ${\sf lpod2asp}(\Pi)$. 

\ii[{\bf (b)}]
Let $K$ be an optimal answer set of ${\sf lpod2asp}(\Pi)$ such that $K \vDash ap(x_1,\dots,x_m)$ for some $x_1,\dots,x_m$. 
By Lemma \ref{prop:lpod2asp:base-3}, there exists an answer set $S'$ of $\Pi_1 \cup \Pi_2$ such that $S'$ and $K$ agrees on the signature of $\Pi_1 \cup \Pi_2$, which means $shrink(S', x_1,\dots,x_m) = shrink(K, x_1,\dots,x_m)$, and $S' \vDash ap(x_1,\dots,x_m)$.
Besides, since $K$ and $S'$ satisfy the same set of $ap(*)$ atoms, and $K$ is an optimal answer set of ${\sf lpod2asp}(\Pi)$, $S'$ is an optimal answer set of $\Pi_1 \cup \Pi_2$.
By Proposition \ref{prop:lpod2asp:base}, $S = shrink(S', x_1, \dots, x_m)$ is a candidate answer set of $\Pi$, and $\phi(S)$ is an answer set of $AP_{\Pi}(x_1, \dots, x_m)$.

\end{itemize}
\end{proof}

\begin{lemma} \label{lem:lpod:preferred}
Under each of the four preference criteria, the preferred answer sets of an LPOD $\Pi$ of signature $\sigma$ are exactly the preferred answer sets on $\sigma$ of ${\sf lpod2asp}(\Pi)$. In other words,
\begin{itemize}
\ii[{\bf (a)}]
for any preferred answer set $S$ of $\Pi$, there exists an optimal answer set $K$ of ${\sf lpod2asp}(\Pi)$ and there are $x_1, \dots, x_m$ such that $K \vDash pAS(x_1, \dots, x_m)$ and $S = shrink(K, x_1, \dots, x_m)$;
\ii[{\bf (b)}]
for any optimal answer set $K$ of ${\sf lpod2asp}(\Pi)$ and any $x_1,\dots,x_m$ such that $K \vDash pAS(x_1, \dots, x_m)$, $S = shrink(K, x_1, \dots, x_m)$ is a preferred answer set of $\Pi$.
\end{itemize}
\end{lemma}

\begin{proof}
({\bf a}) 
Let $\Pi$ be an LPOD (\ref{program:lpod}) of signature $\sigma$. 
Let $S$ be a preferred answer set of $\Pi$; and let $S_2$ be any candidate answer set of $\Pi$ with different satisfaction degrees compared to $S$. 
For any set of atoms $S'$, let $\phi(S') = S' \cup \{ body_i \mid$ $S'$ satisfies the body of rule $i$ in $\Pi_{od}\}$.
By Proposition \ref{prop:lpod2asp:hardsplit}, we know $\phi(S)$ is an answer set of $AP_{\Pi}(x_1, \dots, x_m)$ for some $x_1,\dots,x_m$, and $\phi(S_2) = S_2 \cup \{ body_i \mid$ $S_2$ satisfies the body of rule $i$ for some $1\leq i \leq m$ $\}$ is an answer set of $AP_{\Pi}(y_1, \dots, y_m)$ for some $y_1, \dots, y_m$, where by Lemma \ref{prop:lpod2asp:hardsplit:degree}, the list $x_1, \dots, x_m$ is not the same as $y_1, \dots, y_m$. 

By Lemma \ref{lem:lpod:twoAS} {\bf (a)}, there exists an optimal answer set $K$ of ${\sf lpod2asp}(\Pi)$ such that $K \vDash ap(x_1, \dots, x_m)$,  $K \vDash ap(y_1, \dots, y_m)$, $S = shrink(K, x_1, \dots, x_m)$, and $S_2 = shrink(K, y_1, \dots, y_m)$.

Then it is sufficient to prove $K \vDash pAS(x_1, \dots, x_m)$, which by rules \eqref{equ:lpod:card:pAS}, \eqref{equ:lpod:incl:pAS}, \eqref{equ:lpod:pare:pAS}, \eqref{equ:lpod:pena:pAS}, suffices to proving $K \not \vDash prf(ap(y_1, \dots, y_m), ap(x_1, \dots, x_m))$ no matter what $S_2$ we are choosing. 
Assume for the sake of contradiction that $K \vDash prf(ap(y_1, \dots, y_m), ap(x_1, \dots, x_m))$, 
we will derive a contradiction for each preference criterion. 
Note that
\begin{itemize}
\item
$K \vDash degree(ap(x_1, \dots, x_m), d_1,\dots, d_m)$
\end{itemize}
iff (by rules \eqref{equ:lpod:onedegree}, \eqref{equ:lpod:degreenotbody}, \eqref{equ:lpod:degreebody}, and given $K \vDash ap(x_1, \dots, x_m)$)
\begin{itemize}
\item
for $1\leq i \leq m$, $d_i = 1$ if $x_i=0$, $d_i=x_i$ if $x_i>0$
\end{itemize}
iff (by Lemma \ref{prop:lpod2asp:hardsplit:degree}, and given $S$ is a candidate answer set of $\Pi$, and given $\phi(S)$ is an answer set of $AP_{\Pi}(x_1,\dots,x_m)$)
\begin{itemize}
\item
the satisfaction degrees of $S$ are $d_1, \dots, d_m$.
\end{itemize}
Similarly,
\begin{itemize}
\item
$K \vDash degree(ap(y_1, \dots, y_m), e_1, \dots, e_m)$
\end{itemize}
iff
\begin{itemize}
\item
the satisfaction degrees of $S_2$ are $e_1, \dots, e_m$.
\end{itemize}

\begin{enumerate}
\item
{\bf Cardinality-preferred: }
\begin{itemize}
\item
$K \vDash prf(ap(y_1, \dots, y_m), ap(x_1, \dots, x_m))$
\end{itemize}
iff (by rule (\ref{equ:lpod:card:prf}))
\begin{itemize}
\item
there exists a number $d$ such that $0\leq d \leq maxdegree-1$ and 
  \begin{itemize}
  \item
  $K \vDash prf2degree(ap(y_1, \dots, y_m), ap(x_1, \dots, x_m), d+1)$
  \item
  $K \vDash equ2degree(ap(y_1, \dots, y_m), ap(x_1, \dots, x_m),Y)$ for $1\leq Y \leq d$
  \end{itemize}
\end{itemize}
iff (by rules \eqref{equ:lpod:card:equ2}, \eqref{equ:lpod:card:prf2})
\begin{itemize}
\item
there exists a number $d$ such that $0\leq d \leq maxdegree-1$ and 
  \begin{itemize}
  \item
  there exist $n_1$ and $n_2$ such that
  $K \vDash card(ap(y_1, \dots, y_m),d+1,n_1)$, \\ $K \vDash card(ap(x_1, \dots, x_m),d+1,n_2)$, and $n_1 > n_2$
  \item
  for each $1\leq Y \leq d$, there exists a number $n$ such that 
  $K \vDash card(ap(y_1, \dots, y_m),Y,n)$ and $K \vDash card(ap(x_1, \dots, x_m),Y,n)$
  \end{itemize}
\end{itemize}
iff (by rule (\ref{equ:lpod:card:card}))
\begin{itemize}
\item
there exists a number $d$ such that $0\leq d \leq maxdegree-1$ and 
  \begin{itemize}
  \item
  there exist $n_1$ and $n_2$ such that
  $S_2$ satisfies $n_1$ rules at degree $d$, $S$ satisfies $n_2$ rules at degree $d+1$, and $n_1 > n_2$ 
  \item
  for each $1\leq Y \leq d$, there exists a number $n$ such that 
  both $S_2$ and $S$ satisfy $n$ rules at degree $Y$
  \end{itemize}
\end{itemize}
iff (by the semantics of LPOD)
\begin{itemize}
\item
$S_2$ is cardinality-preferred to $S$
\end{itemize}
which violates the fact that $S$ is a preferred answer set.

\item
{\bf Inclusion-preferred: }
\begin{itemize}
\item
$K \vDash prf(ap(y_1, \dots, y_m), ap(x_1, \dots, x_m))$
\end{itemize}
iff (by rule (\ref{equ:lpod:incl:prf}))
\begin{itemize}
\item
there exists a number $d$ such that $0\leq d \leq maxdegree-1$ and 
  \begin{itemize}
  \item
  $K \vDash prf2degree(ap(y_1, \dots, y_m), ap(x_1, \dots, x_m), d+1)$
  \item
  $K \vDash equ2degree(ap(y_1, \dots, y_m), ap(x_1, \dots, x_m),Y)$ for $1\leq Y \leq d$
  \end{itemize}
\end{itemize}
iff (by rules \eqref{equ:lpod:incl:even}, \eqref{equ:lpod:incl:equ2}, \eqref{equ:lpod:incl:prf2})
\begin{itemize}
\item
there exists a number $d$ such that $0\leq d \leq maxdegree-1$ and 
  \begin{itemize}
  \item
  $K \not \vDash equ2degree(ap(y_1, \dots, y_m), ap(x_1, \dots, x_m),d+1)$ and for $1 \leq i \leq m$, whenever $S$ satisfies rule $i$ at degree $d+1$, $S_2$ must also satisfy rule $i$ at degree $d+1$;
  \footnote{
  The atom $\{D_{11}\neq X; D_{21}=X \}1$ is true in $K$ iff the number of atoms in this set that is satisfied by $K$ is smaller or equal to 1, which means that this atom is true iff $K\vDash \neg (\{D_{11}\neq X \land D_{21}=X)$ iff $K \vDash (D_{21}=X \rightarrow \{D_{11}=X)$. In the case $X=d+1$, this atom is true iff ``whenever $S_2$ satisfies rule 1 at degree $d+1$, $S$ must satisfies rule 1 at degree $d+1$''.
  }
  \item
  for each $1\leq Y \leq d$, $S$ satisfies rule $i$ at degree $Y$ iff $S_2$ satisfies rule $i$ at degree $Y$ for $1\leq i \leq m$
  \footnote{
  The atom $C_1 = \{ D_{11}=X; D_{21}=X \}$ is true in $K$ iff $C_1$ is the number of atoms in this set that is satisfied by $K$. Then $C_1=0 \lor C_1=2$ iff $D_{11}=X \leftrightarrow D_{21}=X$, which can be read as ``$S$ satisfies rule 1 at degree $X$ iff $S_2$ satisfies rule 1 at degree $X$''.
  }
  \end{itemize}
\end{itemize}
iff 
\begin{itemize}
\item
there exists a number $d$ such that $0\leq d \leq maxdegree-1$ and 
  \begin{itemize}
  \item
  the rules satisfied by $S$ is a proper subset of the rules satisfied by $S_2$ at degree $d+1$
  \item
  the rules satisfied by $S$ is exactly the rules satisfied by $S_2$ at degrees $\{1,\dots, d\}$
  \end{itemize}
\end{itemize}
iff (by the semantics of LPOD)
\begin{itemize}
\item
$S_2$ is inclusion-preferred to $S$
\end{itemize}
which violates the fact that $S$ is a preferred answer set.

\item
{\bf Pareto-preferred: }
\begin{itemize}
\item
$K \vDash prf(ap(y_1, \dots, y_m), ap(x_1, \dots, x_m))$
\end{itemize}
iff (by rule (\ref{equ:lpod:pare:prf}))
\begin{itemize}
\item
there exists 2 lists $e_1, \dots, e_m$ and $d_{1}, \dots, d_{m}$ such that
  \begin{itemize}
  \item
  $K \vDash degree(ap(y_1, \dots, y_m), e_1, \dots, e_m)$
  \item
  $K\vDash degree(ap(x_1, \dots, x_m), d_1, \dots, d_m)$
  \item
  $K\not \vDash equ(ap(y_1, \dots, y_m), ap(x_1, \dots, x_m))$, and
  \item
  $e_1 \leq d_1, \dots, e_m \leq d_m$
  \end{itemize}
\end{itemize}
iff (by rule (\ref{equ:lpod:pare:equ}))
\begin{itemize}
\item
there exists 2 lists $e_1, \dots, e_m$ and $d_1, \dots, d_m$ such that
  \begin{itemize}
  \item
  $K \vDash degree(ap(y_1, \dots, y_m), e_1, \dots, e_m)$
  \item
  $K\vDash degree(ap(x_1, \dots, x_m), d_1, \dots, d_m)$
  \item
  $e_1 \leq d_1, \dots, e_m \leq d_m$, and there exists an $i$ such that $e_i < d_i$
  \end{itemize}
\end{itemize}
iff (by the semantics of LPOD)
\begin{itemize}
\item
$S_2$ is Pareto-preferred to $S$
\end{itemize}
which violates the fact that $S$ is a preferred answer set.

\item
{\bf Penalty-Sum-preferred: }
\begin{itemize}
\item
$K \vDash prf(ap(y_1, \dots, y_m), ap(x_1, \dots, x_m))$
\end{itemize}
iff (by rule (\ref{equ:lpod:pena:prf}))
\begin{itemize}
\item
there exist $n_1$ and $n_2$ such that
  \begin{itemize}
  \item
  $K \vDash sum(ap(y_1, \dots, y_m), n_1)$
  \item
  $K\vDash sum(ap(x_1, \dots, x_m), n_2)$, and
  \item
  $n_1 < n_2$
  \end{itemize}
\end{itemize}
iff (by rule (\ref{equ:lpod:pena:sum}))
\begin{itemize}
\item
there exist $n_1$ and $n_2$ such that
  \begin{itemize}
  \item
  the sum of the satisfaction degrees of all rules for $S_2$ is $n_1$
  \item
  the sum of the satisfaction degrees of all rules for $S$ is $n_2$, and
  \item
  $n_1 < n_2$
  \end{itemize}
\end{itemize}
iff (by the semantics of LPOD)
\begin{itemize}
\item
$S_2$ is penalty-sum-preferred to $S$
\end{itemize}
which violates the fact that $S$ is a preferred answer set.

\end{enumerate}

({\bf b}) 
Let $\Pi$ be an LPOD (\ref{program:lpod}) of signature $\sigma$; let $K$ be an optimal answer set of ${\sf lpod2asp}(\Pi)$; and let $K$ satisfy $pAS(x_1,\dots, x_m)$. By rules \eqref{equ:lpod:card:pAS}, \eqref{equ:lpod:incl:pAS}, \eqref{equ:lpod:pare:pAS}, \eqref{equ:lpod:pena:pAS}, $K\vDash ap(x_1,\dots, x_m)$. By Lemma \ref{lem:lpod:candidate}, $S=shrink(K, x_1, \dots, x_m)$ is an answer set of $AP_{\Pi}(x_1, \dots, x_m)$. We will prove that $S$ is a preferred answer set of $\Pi$. 

Assume for the sake of contradiction that there exists a candidate answer set $S_2$ of $\Pi$ and $S_2$ is preferred to $S$. By Proposition \ref{prop:lpod2asp:hardsplit}, $S_2$ is also an answer set of $AP_{\Pi}(y_1, \dots, y_m)$ for some $y_1, \dots, y_m$, where by Lemma \ref{prop:lpod2asp:hardsplit:degree}, the list $y_1, \dots, y_m$ is not the same as $x_1, \dots, x_m$. 
By Lemma \ref{lem:lpod:twoAS} {\bf (b)}, $K$ must satisfy $ap(y_1, \dots, y_m)$. Since $K \vDash pAS(x_1,\dots,x_m)$, by rules \eqref{equ:lpod:card:pAS}, \eqref{equ:lpod:incl:pAS}, \eqref{equ:lpod:pare:pAS}, \eqref{equ:lpod:pena:pAS}, to prove a contradiction, it is sufficient to prove $K \vDash$ $prf(ap(y_1, \dots, y_m),$ $ap(x_1, \dots, x_m))$. 

By Lemma \ref{lem:lpod:candidate}, $shrink(K, y_1, \dots, y_m)$ is a candidate answer set of $\Pi$. By Lemma \ref{lem:lpod:candidate} and Lemma \ref{prop:lpod2asp:hardsplit:degree}, $shrink(K, y_1, \dots, y_m)$ has the same satisfaction degrees as $S_2$. So $shrink(S', y_1, \dots, y_m)$ is preferred to $S$. 
As we proved in bullet {\bf (a)}, under any of the four criterion, $shrink(S', y_1, \dots, y_m)$ is preferred to $S$ iff $K \vDash prf(ap(y_1, \dots, y_m),$ $ap(x_1, \dots, x_m))$. Since $shrink(S', y_1, \dots, y_m)$ is preferred to $S$, $K \vDash prf(ap(y_1, \dots, y_m),$ $ap(x_1, \dots, x_m))$.
\end{proof}

\bigskip

\noindent{\bf Theorem~\ref{thm:lpod2asp} \optional{thm:lpod2asp}}\ \ 
{\sl
Under any of the four preference criteria, the candidate (preferred, respectively) answer sets of an LPOD $\Pi$ of signature $\sigma$ are exactly the candidate (preferred, respectively) answer sets on $\sigma$ of ${\sf lpod2asp}(\Pi)$.
}

\begin{proof}
The proof follows from Lemma~\ref{lem:lpod:candidate} and Lemma~\ref{lem:lpod:preferred}.
\end{proof}

\section{Proof of Proposition \ref{prop:crp:SP} } \label{sec:prop:crp:SP}

Let's review the definition of $AP_{\Pi}(x_1, \dots, x_m)$. 
Let $\Pi$ be a $\crpt$ program of signature $\sigma$, where its rules are rearranged such that the cr-rules are of indices $1,\dots, k$, the ordered cr-rules are of indices $k+1,\dots, l$, and the ordered rules are of indices $l+1,\dots, m$. These 3 sets of rules are called $\Pi_{cr}$, $\Pi_{ocr}$, $\Pi_{or}$ respectively, and the remaining part in $\Pi$ is called $\Pi_{r}$. For each rule $i$ in $\Pi_{ocr} \cup \Pi_{or}$, let $n_i$ denote the number of atoms in $head(i)$. 
Let $D_i$ be the set $\{ 0,1 \}$ for $1\leq i \leq k$; $\{ 0,\dots,  n_i\}$ for $k+1 \leq i \leq l$; $\{ 1,\dots,n_i \}$ for $l+1 \leq i \leq m$.
$AP_{\Pi}(x_1, \dots, x_m)$ denotes an assumption program obtained from $\Pi$ as follows, where $x_i \in D_i$.
%
\begin{itemize}
\item
$AP_{\Pi}(x_1, \dots, x_m)$ contains $\Pi_r$

\item 
for each cr-rule
$
i: \i{Head}_i \stackrel{+}\leftarrow \i{Body}_i
$
in $\Pi_{cr}$, $AP_{\Pi}(x_1, \dots, x_m)$ contains
\beq
\i{Head}_i \leftarrow \i{Body}_i, x_i=1
\eeq {proof:crp:SP:1}

\item
for each ordered rule or ordered cr-rule
$
i: C^1_i \times \dots \times C^{n_i}_i \stackrel{(+)}\leftarrow \i{Body}_i
$
in $\Pi_{or} \cup \Pi_{ocr}$, 
for $1\leq j \leq n_i$, $AP_{\Pi}(x_1, \dots, x_m)$ contains
\beq
C^j_i \leftarrow \i{Body}_i, x_i=j 
\eeq {proof:crp:SP:2}

\item
$AP_{\Pi}(x_1, \dots, x_m)$ also contains the following rules: 
\[
\ba {l}
\i{isPreferred}(R1,R2) \leftarrow \i{prefer}(R1,R2). \\
\i{isPreferred}(R1,R3) \leftarrow \i{prefer}(R1,R2), \i{isPreferred}(R2,R3). \\
\leftarrow \i{isPreferred}(R,R). \\
\leftarrow x_{r_1}>0, x_{r_2}>0, \i{isPreferred}(r_1,r_2). \ \ \ \ \ \ (1 \leq r_1, r_2 \leq l)
\ea
\]
\end{itemize}

\bigskip

\noindent{\bf Proposition~\ref{prop:crp:SP}}\ \ 
{\sl
For any $\crpt$ program $\Pi$ of signature $\sigma$, a set $X$ of atoms is the projection of a generalized answer set of $\Pi$ onto $\sigma$ iff $X$ is the projection of an answer set of an assumption program of $\Pi$ onto $\sigma$.
In other words,
\begin{itemize}
\ii[{\bf (a)}]
for any generalized answer set $S$ of $\Pi$, there exists an assumption program $AP_{\Pi}(x_1, \dots, x_m)$ of $\Pi$ and one of its answer set $S'$ such that $S|_{\sigma} = S'|_{\sigma}$;
\ii[{\bf (b)}]
for any answer set $S'$ of any assumption program $AP_{\Pi}(x_1, \dots, x_m)$ of $\Pi$, there exists a generalized answer set $S$ of $\Pi$ such that $S'|_{\sigma} = S|_{\sigma}$.
\end{itemize}
}

\begin{proof}
Let $\Pi$ be a $\crpt$ program.
According to the semantics of $\crpt$, $S$ is a generalized answer set of $\Pi$ iff $S$ is an answer set of $H'_{\Pi}$, where $H'_{\Pi}$ is obtained from $\Pi$ as follows.
\footnote{
Note that $H'_{\Pi}$ is similar to $H_{\Pi}$ (which is defined in Section \ref{subsec:review:crp} of the paper) except that $H'_{\Pi}$ contains a choice rule $\{ A \}$ for each $A \in atoms(H_{\Pi}, \{appl\})$.
}
\begin{itemize}
\item
$H'_{\Pi}$ contains $\Pi_r$

\item 
for each cr-rule
$
i: \i{Head}_i \stackrel{+}\leftarrow \i{Body}_i
$
in $\Pi_{cr}$, $H'_{\Pi}$ contains
\beq
\i{Head}_i \leftarrow \i{Body}_i, \i{appl}(i)
\eeq {proof:crp:pi':1}

\item
for each ordered cr-rule
$
i: C^1_i \times \dots \times C^{n_i}_i \stackrel{+}\leftarrow \i{Body}_i
$
in $\Pi_{ocr}$, 
for $1\leq j \leq n_i$, $H'_{\Pi}$ contains
\begin{align} 
&C^j \leftarrow \i{Body}_i, \i{appl}(i), \i{appl}(choice(i, j))  \label{proof:crp:pi':2}\\
&\i{fired}(i) \leftarrow \i{appl}(choice(i, j))  \label{proof:crp:pi':3}\\
&\i{prefer}(choice(i, j), choice(i, j+1)) ~~~~~(1\leq j \leq n-1)  \label{proof:crp:pi':4}\\
&\leftarrow \i{Body}_i, \i{appl}(i), \no\ \i{fired}(i)  \label{proof:crp:pi':5}
\end{align}

\item
for each ordered rule
$
i: C^1_i \times \dots \times C^{n_i}_i \leftarrow \i{Body}_i
$
in $\Pi_{or}$, 
for $1\leq j \leq n$, $H'_{\Pi}$ contains
\begin{align} 
&C^j \leftarrow \i{Body}_i, \i{appl}(choice(i, j)) \label{proof:crp:pi':6}\\
&\i{fired}(i) \leftarrow \i{appl}(choice(i, j)) \label{proof:crp:pi':7}\\
&\i{prefer}(choice(i, j), choice(i, j+1)) ~~~~~(1\leq j \leq n-1) \label{proof:crp:pi':8}\\
&\leftarrow \i{Body}_i, \no\ \i{fired}(i) \label{proof:crp:pi':9}
\end{align}

\item
$H'_{\Pi}$ also contains:
\begin{align} 
&\i{isPreferred}(R1,R2) \leftarrow \i{prefer}(R1,R2). \label{proof:crp:pi':10}\\
&\i{isPreferred}(R1,R3) \leftarrow \i{prefer}(R1,R2), \i{isPreferred}(R2,R3). \label{proof:crp:pi':11}\\
&\leftarrow \i{isPreferred}(R,R). \label{proof:crp:pi':12}\\
&\leftarrow \i{appl}(R1), \i{appl}(R2), \i{isPreferred}(R1,R2). \label{proof:crp:pi':13}
\end{align}

\item
and for each $A \in atoms(H_{\Pi}, \{appl\})$, $H'_{\Pi}$ also contains
\beq
\{ A \}.
\eeq {proof:crp:pi':14}
\end{itemize}

Note that rule \eqref{proof:crp:pi':10} can be considered as two rules: (\ref{proof:crp:pi':10}r), in which each variable is grounded by an index of a cr-rule; and (\ref{proof:crp:pi':10}a), in which each variable is grounded by a term $choice(*)$. Similarly, each of the rules \eqref{proof:crp:pi':11}, \eqref{proof:crp:pi':12}, \eqref{proof:crp:pi':13} can be considered as two rules. 

The (propositional) signature of $H'_{\Pi}$ is $\sigma \cup atoms(H'_{\Pi}, \{appl, fired, prefer, isPreferred \})$, while the (propositional) signature of $AP_{\Pi}(x_1,\dots,x_m)$ is $\sigma \cup atoms(AP_{\Pi}(x_1,\dots,x_m), \{isPreferred\})$, which is a subset of the signature of $H'_{\Pi}$.

\begin{itemize}
\ii[{\bf (a)}]
Let $S$ be a generalized answer set of $\Pi$. Then $S$ is an answer set of $H'_{\Pi}$.
We obtain $x_1, \dots, x_m$ such that
\begin{itemize}
\item
for $1\leq i \leq k$:
$x_i = 0$ if $S\not \vDash \i{appl}(i)$,

\hspace{2.1cm}$x_i = 1$ if $S \vDash \i{appl}(i)$;
\item
for $k+1\leq i \leq l$:
$x_i = 0$ if $S\not \vDash \i{appl}(i)$,

\hspace{2.7cm}$x_i = j$ if $S \vDash \i{appl}(i)$ and $S \vDash \i{appl}(choice(i,j))$, 
\footnote{
Since $S$ is an answer set of $H'_{\Pi}$, by rules \eqref{proof:crp:pi':4}, \eqref{proof:crp:pi':10}, \eqref{proof:crp:pi':11}, and \eqref{proof:crp:pi':13}, $S$ cannot satisfy $\i{appl}(choice(i, j))$ for two different $j$.
}

\hspace{2.7cm}$x_i = 1$ if $S \vDash \i{appl}(i)$ and $S \not \vDash \i{appl}(choice(i,j))$ for all $j$ (in the 

\hspace{2.7cm}case when $S \not \vDash \i{Body}_i$);
\item
for $l+1\leq i \leq m$: 
$x_i = j$ if $S \vDash \i{appl}(choice(i,j))$,

\hspace{2.8cm}$x_i = 1$ if $S \not \vDash \i{appl}(choice(i,j))$ for all $j$.
\end{itemize}
Then it is sufficient to prove that the projection of $S$ onto 
$$\sigma \cup atoms(AP_{\Pi}(x_1,\dots,x_m), \{\i{isPreferred}\})$$ 
is an answer set of $AP_{\Pi}(x_1, \dots, x_m)$. This is equivalent to proving $S$ is a minimal model of the reduct of $AP_{\Pi}(x_1, \dots, x_m)$ relative to $\sigma \cup atoms(AP_{\Pi}(x_1,\dots,x_m), \{\i{isPreferred}\})$.

The assumption program $AP_{\Pi}(x_1, \dots, x_m)$ is similar to $H'_{\Pi}$ except that 
\begin{enumerate}
\item
$AP_{\Pi}(x_1, \dots, x_m)$ does not contain the constraints: \eqref{proof:crp:pi':5}, \eqref{proof:crp:pi':9}, (\ref{proof:crp:pi':12}a,) (\ref{proof:crp:pi':13}a) 
\item
$AP_{\Pi}(x_1, \dots, x_m)$ does not contain the definitions for $fired(*)$, $prefer(choice(*), choice(*))$, and $isPreferred(choice(*), choice(*))$: \eqref{proof:crp:pi':3}, \eqref{proof:crp:pi':4}, \eqref{proof:crp:pi':7}, \eqref{proof:crp:pi':8}, (\ref{proof:crp:pi':10}a), (\ref{proof:crp:pi':11}a)
\item
$AP_{\Pi}(x_1, \dots, x_m)$ uses the value assignments for $x_i$ to represent $appl(*)$ in $H'_{\Pi}$
\end{enumerate}

Let $(H'_{\Pi})_{i,\dots,j}$ denote the set of rules in $H'_{\Pi}$ translated by rules $(i),\dots,(j)$.

First, let's obtain $\Pi_1$ from $H'_{\Pi}$ by removing the constraints \eqref{proof:crp:pi':5}, \eqref{proof:crp:pi':9}, (\ref{proof:crp:pi':12}a,) (\ref{proof:crp:pi':13}a). In other words,  $\Pi_1$ is $H'_{\Pi} \setminus (H'_{\Pi})_{\ref{proof:crp:pi':5}, \ref{proof:crp:pi':9}, \ref{proof:crp:pi':12}a, \ref{proof:crp:pi':13}a}$.
By Lemma \ref{lem:1} (e), $S$ is an answer set of $\Pi_1$.

Second, let's obtain $\Pi_2$ from $\Pi_1$ by removing the definitions for $fired(*)$,\\ $prefer(choice(*), choice(*))$, and $isPreferred(choice(*), choice(*))$. In other words, $\Pi_2$ is $\Pi_1 \setminus (H'_{\Pi})_{\ref{proof:crp:pi':3}, \ref{proof:crp:pi':4}, \ref{proof:crp:pi':7}, \ref{proof:crp:pi':8}, \ref{proof:crp:pi':10}a, \ref{proof:crp:pi':11}a}$. Let $\sigma_1$ be the propositional signature of $\Pi_1$ and let $\sigma_2$ be the propositional signature of $\Pi_2$. We will use the splitting theorem to split $\Pi_1$ into $\Pi_2$ and $(H'_{\Pi})_{\ref{proof:crp:pi':3}, \ref{proof:crp:pi':4}, \ref{proof:crp:pi':7}, \ref{proof:crp:pi':8}, \ref{proof:crp:pi':10}a, \ref{proof:crp:pi':11}a}$.
Since 
\begin{enumerate}
\item
no atom in $\sigma_2$ has a strictly positive occurrence in $(H'_{\Pi})_{\ref{proof:crp:pi':3}, \ref{proof:crp:pi':4}, \ref{proof:crp:pi':7}, \ref{proof:crp:pi':8}, \ref{proof:crp:pi':10}a, \ref{proof:crp:pi':11}a}$, 
\item
no atom in $\sigma_1 \setminus \sigma_2$ has a strictly positive occurrence in $\Pi_2$, and 
\item
each strongly connected component of the dependency graph of $\Pi_1$ w.r.t. $\sigma_1$ is a subset of $\sigma_2$ or $\sigma_1 \setminus \sigma_2$, 
\end{enumerate}
by the splitting theorem, $S$ is an answer set of $\Pi_2$ relative to $\sigma_2$, where $\sigma_2$ equals to $\sigma \cup atoms(\Pi_2, \{ appl\}) \cup atoms(\Pi_2, \{\i{isPreferred}\})$.

Third, by the assignments of $x_i,\dots, x_m$, we know 
\begin{itemize}
\item
for $1\leq i \leq k$:
$S \vDash \i{appl}(i)$ iff $x_i = 1$,
\item
for $k+1\leq i \leq l$:
$S \vDash \i{Body}_i \land \i{appl}(i) \land \i{appl}(choice(i,j))$ iff $S\vDash \i{Body}_i$ and $x_i = j$
\item
for $l+1\leq i \leq m$: 
$S \vDash \i{Body}_i \land \i{appl}(choice(i,j))$ iff $S\vDash \i{Body}_i$ and $x_i = j$.
\end{itemize}
Note that we can obtain $AP_{\Pi}(x_1,\dots,x_m)$ from $\Pi_2$ by 
\begin{itemize}
\item
for $1\leq i \leq k$, replacing $\i{appl}(i)$ with $x_i = 1$ in rule \eqref{proof:crp:pi':1};

\item
for $k+1 \leq i \leq l$, replacing $\i{appl}(i) \land \i{appl}(choice(i,j))$ with $x_i = j$ in rule \eqref{proof:crp:pi':2};

\item
for $l+1 \leq i \leq m$, replacing $\i{appl}(choice(i,j))$ with $x_i = j$ in rule \eqref{proof:crp:pi':6}

\item
for $1 \leq i \leq l$, replacing $\i{appl}(i)$ with $x_i > 0$ in (grounded) rule \eqref{proof:crp:pi':13}.
\end{itemize}

Since $S$ is a minimal model of the reduct of $\Pi_2$ relative to $\sigma \cup atoms(H_{\Pi}, appl) \cup atoms(\Pi_2, \{\i{isPreferred}\})$, $S$ is a minimal model of the reduct of $AP_{\Pi}(x_1,\dots,x_m)$ relative to $\sigma \cup atoms(\Pi_2, \{\i{isPreferred}\})$. Since 
$$atoms(\Pi_2, \{\i{isPreferred}\}) = atoms(AP_{\Pi}(x_1,\dots,x_m), \{\i{isPreferred}\}),$$
$S$ is a minimal model of the reduct of $AP_{\Pi}(x_1,\dots,x_m)$ relative to $$\sigma \cup atoms(AP_{\Pi}(x_1,\dots,x_m), \{\i{isPreferred}\}). $$

\ii[{\bf (b)}]
Let $AP_{\Pi}(x_1,\dots,x_m)$ be an assumption program of $\Pi$, and $S_{sp}$ be an answer set of $AP_{\Pi}(x_1,\dots,x_m)$. 
\[
\ba {rl}
\text{Let  } S = S_{sp} & \cup \{ appl(i) \mid  1\leq i \leq k, x_i=1 \} \\
& \cup \{ appl(i), appl(choice(i,j)), fired(i) \mid  k+1\leq i \leq l, x_i =j, j>0 \} \\
& \cup \{ appl(choice(i,j)), fired(i)  \mid  l+1\leq i \leq m, x_i =j \} \\
& \cup \{ prefer(choice(i,j), choice(i,j+1)) \mid  k+1\leq i \leq m, 1\leq j\leq n_i \} \\
& \cup \{ isPreferred(choice(i,j_1), choice(i,j_2)) \mid  k+1\leq i \leq m, 1\leq j_1 < j_2\leq n_i \} 
\ea
\]
It is sufficient to prove $S$ is an answer set of $H'_{\Pi}$.

Let $\Pi_1$ be $H'_{\Pi} \setminus (H'_{\Pi})_{\ref{proof:crp:pi':5}, \ref{proof:crp:pi':9}, \ref{proof:crp:pi':12}a, \ref{proof:crp:pi':13}a}$. Let $\Pi_2$ be $\Pi_1 \setminus (H'_{\Pi})_{\ref{proof:crp:pi':3}, \ref{proof:crp:pi':4}, \ref{proof:crp:pi':7}, \ref{proof:crp:pi':8}, \ref{proof:crp:pi':10}a, \ref{proof:crp:pi':11}a}$. 

First, we prove
\[
\ba {rl}
S_{sp} & \cup \{ appl(i) \mid  1\leq i \leq k, x_i=1 \} \\
& \cup \{ appl(i), appl(choice(i,j)) \mid  k+1\leq i \leq l, x_i =j, j>0 \} \\
& \cup \{ appl(choice(i,j))  \mid  l+1\leq i \leq m, x_i =j \},
\ea
\]
denoted by $S_2$, is an answer set of $\Pi_2$. 
Let's compare the reduct of $AP_{\Pi}(x_1,\dots,x_m)$ relative to $S_{sp}$ and the reduct of $\Pi_2$ relative to $S_2$. The reduct of $\Pi_2$ relative to $S_2$ can be obtained from the reduct of $AP_{\Pi}(x_1,\dots,x_m)$ relative to $S_{sp}$ by adding the facts
\begin{enumerate}
\item
$appl(i)$ for $1\leq i \leq k$ and $x_i=1$,
\item
$appl(i)$ and $appl(choice(i,j))$ for $k+1\leq i \leq l$, and $x_i =j, j>0$, 
\item
$appl(choice(i,j))$ for $l+1\leq i \leq m$, and $x_i =j$;
\end{enumerate}
and replacing 
\begin{enumerate}
\item
$x_i=1$ by $appl(i)$ for $1\leq i \leq k$,
\item
$x_i =j$, where $j>0$, by $appl(i) \land appl(choice(i,j))$ for $k+1\leq i \leq l$, 
\item
$x_i =j$ by $appl(choice(i,j))$ for $l+1\leq i \leq m$.
\end{enumerate}
Since $S_{sp}$ is a minimal model of the reduct of $AP_{\Pi}(x_1,\dots,x_m)$ relative to $S_{sp}$, and since
\begin{enumerate}
\item
for $1\leq i \leq k$, $S_2 \vDash appl(i)$ iff $x_i=1$,
\item
for $k+1\leq i \leq l$, $S_2 \vDash appl(i) \land appl(choice(i,j))$  iff $x_i =j \land j>0$, 
\item
for $l+1\leq i \leq m$, $appl(choice(i,j))$ iff $x_i =j$;
\end{enumerate}
$S_2$ is a minimal model of the reduct of $\Pi_2$ relative to $S_2$.

Second, we prove $S$ is an answer set of $\Pi_1$. Note that $S$ equals
\[
\ba {rl}
S_2 & \cup \{ fired(i) \mid  k+1\leq i \leq l, x_i =j, j>0 \} \\
& \cup \{ fired(i)  \mid  l+1\leq i \leq m, x_i =j \} \\
& \cup \{ prefer(choice(i,j), choice(i,j+1)) \mid  k+1\leq i \leq m, 1\leq j\leq n_i \} \\
& \cup \{ isPreferred(choice(i,j_1), choice(i,j_2)) \mid  k+1\leq i \leq m, 1\leq j_1 < j_2\leq n_i \} .
\ea
\]
Let $\sigma_1$ be the propositional signature of $\Pi_1$ and let $\sigma_2$ be the propositional signature of $\Pi_2$. We will use the splitting theorem to construct $\Pi_1$ from $\Pi_2$ and $(H'_{\Pi})_{\ref{proof:crp:pi':3}, \ref{proof:crp:pi':4}, \ref{proof:crp:pi':7}, \ref{proof:crp:pi':8}, \ref{proof:crp:pi':10}a, \ref{proof:crp:pi':11}a}$.
Note that 
\begin{enumerate}
\item
no atom in $\sigma_2$ has a strictly positive occurrence in $(H'_{\Pi})_{\ref{proof:crp:pi':3}, \ref{proof:crp:pi':4}, \ref{proof:crp:pi':7}, \ref{proof:crp:pi':8}, \ref{proof:crp:pi':10}a, \ref{proof:crp:pi':11}a}$, 
\item
no atom in $\sigma_1 \setminus \sigma_2$ has a strictly positive occurrence in $\Pi_2$, and 
\item
each strongly connected component of the dependency graph of $\Pi_1$ w.r.t. $\sigma_1$ is a subset of $\sigma_2$ or $\sigma_1 \setminus \sigma_2$, 
\end{enumerate}
Since $S$ is an answer set of $\Pi_2$ relative to $\sigma_2$, and it's easy to check that $S$ is an answer set of $(H'_{\Pi})_{\ref{proof:crp:pi':3}, \ref{proof:crp:pi':4}, \ref{proof:crp:pi':7}, \ref{proof:crp:pi':8}, \ref{proof:crp:pi':10}a, \ref{proof:crp:pi':11}a}$ relative to $\sigma_1 \setminus \sigma_2$, $S$ is an answer set of $\Pi_1$.


Third, since $S$ satisfies rules \eqref{proof:crp:pi':5}, \eqref{proof:crp:pi':9}, (\ref{proof:crp:pi':12}a), (\ref{proof:crp:pi':13}a), by Lemma \ref{lem:1} (d), $S$ is an answer set of $H'_{\Pi}$.

\end{itemize}
\end{proof}

\section{Proof of Theorem \ref{thm:crp2asp} } \label{sec:proof:thm:crp2asp}

We first review some definitions. Let $\Pi$ be a $\crpt$ program.
Let $S$ be an optimal answer set of ${\sf crp2asp}(\Pi)$. 
Let $x_1, \dots, x_m$ be a list of integers such that $x_i \in D_i$.
If $S \vDash ap(x_1, \dots, x_m)$, we define the set $shrink(S, x_1, \dots, x_m)$ as a {\em generalized answer set on $\sigma$} of ${\sf crp2asp}(\Pi)$; if $S \vDash candidate(x_1, \dots, x_m)$, we define the set $shrink(S, x_1, \dots, x_m)$ as a {\em candidate answer set on $\sigma$} of ${\sf crp2asp}(\Pi)$; if $S \vDash pAS(x_1, \dots, x_m)$, we define the set $shrink(S, x_1, \dots, x_m)$ as a {\em preferred answer set on $\sigma$} of ${\sf crp2asp}(\Pi)$. 

\bigskip
\noindent{\bf Theorem~\ref{thm:crp2asp}}\ \ 
{\sl
For any $\crpt$ program $\Pi$ of signature $\sigma$,
\begin{itemize}
\item[{\bf (a)}]  The projections of the generalized answer sets of $\Pi$ onto $\sigma$ are exactly the generalized answer sets on $\sigma$ of ${\sf crp2asp}(\Pi)$.
\item[{\bf (b)}]  The projections of the candidate answer sets of $\Pi$ onto $\sigma$ are exactly the candidate answer sets on $\sigma$ of ${\sf crp2asp}(\Pi)$.
\item[{\bf (c)}]  The preferred answer sets of $\Pi$ are exactly the preferred answer sets on $\sigma$ of ${\sf crp2asp}(\Pi)$.
\end{itemize}
}

\medskip

\begin{proof}
\noindent{\bf (a): } Let $\Pi$ be a $\crpt$ program of signature $\sigma$. By Proposition \ref{prop:crp:SP}, it is sufficient to prove that the projections (onto $\sigma$) of the answer sets of all assumption programs $AP_{\Pi}(x_1,\dots,x_m)$ of $\Pi$ are exactly the generalized answer sets on $\sigma$ of ${\sf crp2asp}(\Pi)$ such that
\begin{itemize}
\item
for any answer set $S$ of any $AP_{\Pi}(x_1,\dots,x_m)$, there exists an optimal answer set $S'$ of ${\sf crp2asp}(\Pi)$ such that $S' \vDash ap(x_1,\dots,x_m)$ and $S_{\sigma} = shrink(S',x_1,\dots,x_m)$;
\item
for any generalized answer set on $\sigma$, $shrink(S',x_1,\dots,x_m)$, of ${\sf crp2asp}(\Pi)$ (where $S'$ is an optimal answer set of ${\sf crp2asp}(\Pi)$ and $S' \vDash ap(x_1,\dots,x_m)$),  there exists an answer set $S$ of $AP_{\Pi}(x_1,\dots,x_m)$ such that $S_{\sigma} = shrink(S',x_1,\dots,x_m)$.
\end{itemize}

Let ${\sf crp2asp}(\Pi) = \Pi_{base} \cup \Pi_{pref}$, where $\Pi_{pref}$ is the set of rules translated from rules \eqref{crp:4:atom}, \eqref{crp:4:rule:5}, \eqref{crp:5}, \eqref{crp:6}, \eqref{crp:7}. We use Lemma \ref{lem:2} to prove that there is a 1-1 correspondence between the answer sets of ${\sf crp2asp}(\Pi)$ and the answer sets of $\Pi_{base}$, while an answer set of ${\sf crp2asp}(\Pi)$ agrees with the corresponding answer set of $\Pi_{base}$ on the signature of $\Pi_{base}$.
Let's take $\Pi_{base}$ as our current program, $\Pi_{cur}$, and consider including the translation rules in $\Pi_{pref}$ into $\Pi_{cur}$. If we include rules \eqref{crp:4:atom} and \eqref{crp:4:rule:5}, by Lemma \ref{lem:2}, 
there is a 1-1 correspondence between the answer sets of $\Pi_{cur}$ and the answer sets of $\Pi_{base}$. Similarly, we can include rules \eqref{crp:5}, \eqref{crp:6}, \eqref{crp:7} in order into $\Pi_{cur}$, and find that there is a 1-1 correspondence between the answer sets of $\Pi_{base} \cup \Pi_{pref}$ and the answer sets of $\Pi_{base}$, while an answer set of $\Pi_{base} \cup \Pi_{pref}$ agrees with the corresponding answer set of $\Pi_{base}$ on the signature of $\Pi_{base}$.
Since the predicates introduced by $\Pi_{pref}$ are not in $\sigma$, it is sufficient to prove that the projections of the answer sets of all assumption programs $AP_{\Pi}(x_1,\dots,x_m)$ of $\Pi$ onto $\sigma$ are exactly the generalized answer sets on $\sigma$ of $\Pi_{base}$.

According to the translation, the empty set is always an answer set of $\Pi_{base}$, thus there must exist at least one optimal answer set of $\Pi_{base}$. Furthermore, by rule \eqref{crp:1:wc}, the optimal answer set should contain as many $ap(*)$ as possible. 
Let $gr(\Pi_{base}, x_1,\dots, x_m)$ be a partial grounded program obtained from $\Pi_{base}$ by replacing variables $X_1,\dots, X_m$ with $x_1,\dots, x_m$. Since each partial grounded program is disjoint from each other, by the splitting theorem, it is sufficient to prove a 1-1 correspondence $\phi$ between the answer sets of $AP_{\Pi}(x_1, \dots, x_m)$ and the optimal answer sets of $gr(\Pi_{base}, x_1,\dots, x_m)$ such that
\begin{itemize}
\item  [\bf{(a.1)}]
For any answer set $S$ of $AP_{\Pi}(x_1, \dots, x_m)$, $\phi(S) =   \{ a({\bf v}, x_1, \dots, x_m) \mid a({\bf v}) \in S \} \cup \{ ap(x_1, \dots, x_m) \} $ is an optimal answer set of $gr(\Pi_{base}, x_1,\dots, x_m)$.
\item  [\bf{(a.2)}]
For any optimal answer set $S'$ of $gr(\Pi_{base}, x_1,\dots, x_m)$, if $S' \not \vDash ap(x_1, \dots, x_m)$, then $AP_{\Pi}(x_1, \dots, x_m)$ has no answer set; if $S' \vDash ap(x_1, \dots, x_m)$, then 
$$S = \{ a({\bf v}) \mid a({\bf v}, x_1, \dots, x_m) \in S' \} \setminus \{ sp \} $$ 
is an answer set of $AP_{\Pi}(x_1, \dots, x_m)$.
\end{itemize}

To prove bullet {\bf (a.1)}, let $S$ be an answer set of $AP_{\Pi}(x_1, \dots, x_m)$, and let $\phi(S)$ be $\{ a({\bf v}, x_1, \dots, x_m) \mid a({\bf v}) \in S \} \cup \{ ap(x_1, \dots, x_m) \} $. Since $\phi(S)$ satisfies $ap(x_1, \dots, x_m)$, which is the only $ap(*)$ in $gr(\Pi_{base}, x_1,\dots, x_m)$, if we prove $\phi(S)$ is an answer set of $gr(\Pi_{base}, x_1,\dots, x_m)$, $\phi(S)$ must be an optimal answer set of $gr(\Pi_{base}, x_1,\dots, x_m)$. 
Note that, if we ignore the suffix $x_1,\dots,x_m$ in the reduct of $gr(\Pi_{base}, x_1,\dots, x_m)$ relative to $\phi(S)$, it is almost the same as the reduct of $AP_{\Pi}(x_1, \dots, x_m)$ relative to $S$ except that the former has one more atom $sp$. Since $S$ is a minimal model of the reduct of $AP_{\Pi}(x_1, \dots, x_m)$ relative to $S$, and $\phi(S) \vDash ap(x_1,\dots,x_m)$, $\phi(S)$ is a minimal model of the reduct of $gr(\Pi_{base}, x_1,\dots, x_m)$ relative to $\phi(S)$. Thus $\phi(S)$ is an answer set of $gr(\Pi_{base}, x_1,\dots, x_m)$.

To prove bullet {\bf (a.2)}, let $S'$ be an optimal answer set of $gr(\Pi_{base}, x_1,\dots, x_m)$. There are 2 cases as follows.
\begin{enumerate}
\item
$ap(x_1, \dots, x_m) \not \in S'$. We will prove $AP_{\Pi}(x_1, \dots, x_m)$ has no answer set. Assume for the sake of contradiction that there exists an answer set $S$ of $AP_{\Pi}(x_1, \dots, x_m)$, by the bullet {\bf (a.1)} that we just proved, $\phi(S)$ is an optimal answer set of $gr(\Pi_{base}, x_1,\dots, x_m)$. Since $\phi(S) \vDash ap(x_1, \dots, x_m)$, by rule (\ref{crp:1:wc}), it has lower penalty than $S'$, thus $S'$ is not an optimal answer set, which is not the case. So $AP_{\Pi}(x_1, \dots, x_m)$ has no answer set.

\item
$ap(x_1, \dots, x_m) \in S'$. 
Since $S'$ is a minimal model of the reduct of $gr(\Pi_{base}, x_1,\dots, x_m)$, if we remove all occurrence of $ap(x_1,\dots,x_m)$ and $x_1,\dots,x_m$ in both $S'$ and the reduct of $gr(\Pi_{base}, x_1,\dots, x_m)$ relative to $S'$, the set of atoms $S= \{ a({\bf v}) \mid a({\bf v}, x_1, \dots, x_m) \in S' \} \setminus \{ sp \} $ should be a minimal model of the new program, which is the reduct of $AP_{\Pi}(x_1,\dots,x_m)$.
Thus $S$ is an answer set of $AP_{\Pi}(x_1,\dots,x_m)$.
\end{enumerate}

\medskip
\noindent{\bf (b): } To prove Theorem \ref{thm:crp2asp} {\bf (b)}, it is sufficient to prove
\begin{itemize}
\item [\bf{(b.1)}]
for any candidate answer set $S$ of $\Pi$, there exist an optimal answer set $S'$ of ${\sf crp2asp}(\Pi)$ and a list $x_1, \dots, x_m$ such that $S' \vDash candidate(x_1, \dots, x_m)$, and $S_{\sigma} = shrink(S', x_1, \dots, x_m)$;
\item [\bf{(b.2)}]
for any optimal answer set $S'$ of ${\sf crp2asp}(\Pi)$, if $S' \vDash candidate(x_1, \dots, x_m)$, there exists a candidate answer set $S$ of $\Pi$ such that $S_{\sigma} = shrink(S', x_1, \dots, x_m)$.
\end{itemize}

Let $\Pi$ be a $\crpt$ program with signature $\sigma$; $\Pi'$ be its translation ${\sf crp2asp}(\Pi)$. 

To prove bullet {\bf(b.1)}, let $S$ be a candidate answer set of $\Pi$, then by the semantics of $\crpt$, $S$ must be a generalized answer set of $\Pi$. 
We obtain $x_1, \dots, x_m$ such that, 
\begin{itemize}
\item
for $1\leq i \leq k$:
$x_i = 0$ if $S\not \vDash \i{appl}(i)$,

\hspace{2.1cm}$x_i = 1$ if $S \vDash \i{appl}(i)$;
\item
for $k+1\leq i \leq l$:
$x_i = 0$ if $S\not \vDash \i{appl}(i)$,

\hspace{2.7cm}$x_i = j$ if $S \vDash \i{appl}(i)$ and $S \vDash \i{appl}(choice(i,j))$,

\hspace{2.7cm}$x_i = 1$ if $S \vDash \i{appl}(i)$ and $S \not \vDash \i{appl}(choice(i,j))$ for any $j$;
\item
for $l+1\leq i \leq m$: 
$x_i = j$ if $S \vDash \i{appl}(choice(i,j))$,

\hspace{2.8cm}$x_i = 1$ if $S \not \vDash \i{appl}(choice(i,j))$ for any $j$.
\end{itemize}
Note that the signature of $AP_{\Pi}(x_1,\dots,x_m)$ is $\sigma' = \sigma \cup atoms(AP_{\Pi}(x_1,\dots,x_m), \{\i{isPreferred}\})$. 
As we proved in the proof of Proposition \ref{prop:crp:SP}, $S$ is an answer set of $AP_{\Pi}(x_1,\dots,x_m)$ with respect to $\sigma'$. Then $S_{\sigma'}$ is an answer set of $AP_{\Pi}(x_1,\dots,x_m)$. By the first bullet in the proof for Theorem \ref{thm:crp2asp} (a), $\phi(S_{\sigma'}) = \{ a({\bf v}, x_1, \dots, x_m) \mid a({\bf v}) \in S_{\sigma'} \} \cup \{ ap(x_1, \dots, x_m) \}$ is an optimal answer set of $gr(\Pi_{base}, x_1,\dots, x_m)$. Then there exists an optimal answer set $S'$ of $\Pi'$ such that $S'\vDash ap(x_1, \dots, x_m)$ and $S_{\sigma} = shrink(S', x_1, \dots, x_m)$.

Then, it suffices to proving $S' \vDash candidate(x_1,\dots,x_m)$. Assume for the sake of contradiction that $S' \not \vDash candidate(x_1,\dots,x_m)$.

\begin{itemize}
\item
$S' \not \vDash candidate(x_1,\dots,x_m)$
\end{itemize}
iff (by rule \eqref{crp:5})
\begin{itemize}
\item
there exists an $AP$ such that
$S' \vDash dominate(AP, ap(x_1,\dots,x_m))$
\end{itemize}
iff (by rule \eqref{crp:4:atom} and \eqref{crp:4:rule:5})
\begin{itemize}
\item
there exist $i \in \{ k+1,\dots,m \}$ and a list $x_1',\dots,x_m'$  such that
$S' \vDash ap(x_1',\dots,x_m')$,  
$0<x_i'$, and $x_i' < x_i$, or 
\item
there exist $r_1,r_2 \in \{ 1,\dots,l \}$ and a list $x_1',\dots,x_m'$ such that
$S' \vDash ap(x_1',\dots,x_m')$,  $S' \vDash isPreferred(r_1,r_2, x_1',\dots,x_m')$,  $S' \vDash isPreferred(r_1,r_2, x_1,\dots,x_m)$, $x_{r_1}' >0$, and $x_{r_2} >0$
\end{itemize}
iff (by the first 2 bullets in the proof for Theorem \ref{thm:crp2asp} (a) and by the assignments of $x_i$)
\begin{itemize}
\item
there exists $i \in \{ k+1,\dots,m \}$, a generalized answer set $A$, and $x_i, x_i' \in \{ 1,\dots,n_i \}$ such that $A\vDash appl(choice(i,x_i'))$, $S\vDash appl(choice(i,x_i))$, and $x_i' < x_i$
\item
there exist $r_1,r_2 \in \{ 1,\dots,l \}$, and a generalized answer set $A$ such that $A\vDash isPreferred(r_1,r_2)$, $S\vDash isPreferred(r_1,r_2)$, $A\vDash appl(r_1)$, and $S\vDash appl(r_2)$
\end{itemize}
iff (by the definition of dominate)
\begin{itemize}
\item
there exists a generalized answer set $A$ that dominates $S$
\end{itemize}
which contradicts with the fact that $S$ is a candidate answer set. Thus $S' \vDash candidate(x_1,\dots,x_m)$ and $S_{\sigma} = shrink(S', x_1, \dots, x_m)$.

To prove bullet {\bf(b.2)}, let $S'$ be an optimal answer set of $\Pi'$ and $S' \vDash candidate(x_1,\dots,x_m)$ for some list $x_1\dots,x_m$. By rule \eqref{crp:5}, $S' \vDash ap(x_1,\dots,x_m)$. Then by bullet (a), there exists a generalized answer set $S$ of $\Pi$ such that $S_{\sigma} = shrink(S', x_1,\dots, x_m)$. Then it is sufficient to prove $S$ is a candidate answer set of $\Pi$.

Assume for the sake of contradiction that $S$ is not a candidate answer set of $\Pi$, then there must exists a generalized answer set $A$ that dominates $S$. By the ``iff'' statements above, we can derive $S' \not \vDash candidate(x_1,\dots,x_m)$, which leads to a contradiction.

\medskip
\noindent{\bf (c): } 
Let $\Pi$ be a $\crpt$ program with signature $\sigma$; $\Pi'$ be its translation ${\sf crp2asp}(\Pi)$. 
To prove Theorem \ref{thm:crp2asp} (c), it is sufficient to prove 
\begin{itemize}
\item [{\bf(c.1)}]
for any preferred answer set $S$ of $\Pi$, there exists an optimal answer set $S'$ of $\Pi'$ such that $S'\vDash pAS(x_1,\dots,x_m)$ for some $x_1,\dots,x_m$, and $S_{\sigma} = shrink(S', x_1,\dots,x_m)$
\item [{\bf(c.2)}]
for any optimal answer set $S'$ of $\Pi'$, if $S'\vDash pAS(x_1,\dots,x_m)$ for some $x_1,\dots,x_m$, there exists a preferred answer set $S$ of $\Pi$ such that $S_{\sigma} = shrink(S', x_1,\dots,x_m)$.
\end{itemize}

To prove bullet {\bf(c.1)}, let $S$ be a preferred answer set of $\Pi$, then $S$ must be a candidate answer set of $\Pi$. By Theorem \ref{thm:crp2asp} (b), there exists an optimal answer set $S'$ of $\Pi'$ and a list $x_1,\dots,x_m$ such that $S'\vDash candidate(x_1,\dots,x_m)$ and $S_{\sigma} = shrink(S', x_1,\dots,x_m)$. Then it is sufficient to prove $S' \vDash pAS(x_1,\dots,x_m)$. 

Assume for the sake of contradiction that $S' \not \vDash pAS(x_1,\dots,x_m)$.
\begin{itemize}
\item
$S' \not \vDash pAS(x_1,\dots,x_m)$
\end{itemize}
iff (since $S' \vDash candidate(x_1,\dots,x_m)$, and by rule \eqref{crp:7}) 
\begin{itemize}
\item
there exists a $AP$ such that $S' \vDash lessCrRulesApplied(AP, ap(x_1, \dots, x_m))$
\end{itemize} w
iff (by rule \eqref{crp:6})
\begin{itemize}
\item
there exist a list $x_1',\dots,x_m'$ such that $S' \vDash candidate(x_1',\dots,x_m')$, $x_i' \leq x_i$ for $1\leq i \leq m$, and there exists a $j$ such that $x_j' < x_j$
\end{itemize}
iff (since $S'\not\vDash dominate(ap(x_1',\dots,x_m'), ap(x_1,\dots,x_m))$, by rule \eqref{crp:4:atom})
\begin{itemize}
\item
there exist a list $x_1',\dots,x_m'$ such that $S' \vDash candidate(x_1',\dots,x_m')$, $x_i' \leq x_i$ for $1\leq i \leq m$, there exists a $j$ such that $x_j' < x_j$, and for any $x_i' < x_i$, $x_i'=0$
\end{itemize} 
iff (by the assignments of $x_i$)
\begin{itemize}
\item
there exist a candidate answer set $A$ such that the atoms of the form $appl(*)$ in $A$ is a proper subset of those in $S$
\end{itemize} 
which contradicts with the fact that $S$ is a preferred answer set.

To prove bullet {\bf(c.2)}, let $S'$ be an optimal answer set of $\Pi'$ and $S' \vDash pAS(x_1,\dots,x_m)$ for some list $x_1,\dots,x_m$. By rules \eqref{crp:7} and \eqref{crp:5}, $S' \vDash candidate(x_1,\dots,x_m)$ and $S' \vDash ap(x_1,\dots,x_m)$. Then by Theorem \ref{thm:crp2asp} (b), there exists a candidate answer set $S$ of $\Pi$ such that $S_{\sigma} = shrink(S', x_1,\dots, x_m)$. Then it is sufficient to prove $S$ is a preferred answer set of $\Pi$. 

Assume for the sake of contradiction that $S$ is not a preferred answer set of $\Pi$, then there must exists a candidate answer set $A$ such that the atoms of the form $appl(*)$ in $A$ is a proper subset of those in $S$.
By the ``iff'' statements above, we can derive $S' \not \vDash pAS(x_1,\dots,x_m)$, which leads to a contradiction.
\end{proof}

\end{appendix}

%% file: lpod-crprolog-tplp-0501.bbl
\begin{thebibliography}{}

\bibitem[\protect\citeauthoryear{Asuncion, Zhang, and Zhang}{Asuncion
  et~al\mbox{.}}{2014}]{asuncion14logic}
{\sc Asuncion, V.}, {\sc Zhang, Y.}, {\sc and} {\sc Zhang, H.} 2014.
\newblock Logic programs with ordered disjunction: first-order semantics and
  expressiveness.
\newblock In {\em Proceedings of the Fourteenth International Conference on
  Principles of Knowledge Representation and Reasoning}. AAAI Press, 2--11.

\bibitem[\protect\citeauthoryear{Balduccini, , Balduccini, and
  Mellarkod}{Balduccini et~al\mbox{.}}{2003}]{bald03b}
{\sc Balduccini, M.}, {\sc }, {\sc Balduccini, M.}, {\sc and} {\sc Mellarkod,
  V.} 2003.
\newblock {C}{R}-{P}rolog with ordered disjunction.
\newblock In {\em In ASP03 Answer Set Programming: Advances in Theory and
  Implementation, volume 78 of CEUR Workshop proceedings}.

\bibitem[\protect\citeauthoryear{Balduccini}{Balduccini}{2007}]{balduccini07cr-models}
{\sc Balduccini, M.} 2007.
\newblock {CR-MODELS}: an inference engine for {C}{R}-{P}rolog.
\newblock In {\em Proceedings of the 9th international conference on Logic
  programming and nonmonotonic reasoning}. Springer-Verlag, 18--30.

\bibitem[\protect\citeauthoryear{Balduccini and Gelfond}{Balduccini and
  Gelfond}{2003}]{balduccini03logic}
{\sc Balduccini, M.} {\sc and} {\sc Gelfond, M.} 2003.
\newblock Logic programs with consistency-restoring rules.
\newblock In {\em International Symposium on Logical Formalization of
  Commonsense Reasoning, AAAI 2003 Spring Symposium Series}. 9--18.

\bibitem[\protect\citeauthoryear{Balduccini and Mellarkod}{Balduccini and
  Mellarkod}{2004}]{balduccini04a-prolog}
{\sc Balduccini, M.} {\sc and} {\sc Mellarkod, V.} 2004.
\newblock {A}-{P}rolog with {C}{R}-rules and ordered disjunction.
\newblock In {\em Intelligent Sensing and Information Processing, 2004.
  Proceedings of International Conference on}. IEEE, 1--6.

\bibitem[\protect\citeauthoryear{Brewka}{Brewka}{2002}]{brewka02logic}
{\sc Brewka, G.} 2002.
\newblock Logic programming with ordered disjunction.
\newblock In {\em AAAI/IAAI}. 100--105.

\bibitem[\protect\citeauthoryear{Brewka}{Brewka}{2005}]{brewka05preferences}
{\sc Brewka, G.} 2005.
\newblock Preferences in answer set programming.
\newblock In {\em CAEPIA}. Vol. 4177. Springer, 1--10.

\bibitem[\protect\citeauthoryear{Brewka, Delgrande, Romero, and Schaub}{Brewka
  et~al\mbox{.}}{2015}]{brewka15asprin}
{\sc Brewka, G.}, {\sc Delgrande, J.~P.}, {\sc Romero, J.}, {\sc and} {\sc
  Schaub, T.} 2015.
\newblock asprin: Customizing answer set preferences without a headache.
\newblock In {\em AAAI}. 1467--1474.

\bibitem[\protect\citeauthoryear{Brewka, Niemel{\"a}, and Syrj{\"a}nen}{Brewka
  et~al\mbox{.}}{2002}]{brewka02implementing}
{\sc Brewka, G.}, {\sc Niemel{\"a}, I.}, {\sc and} {\sc Syrj{\"a}nen, T.} 2002.
\newblock Implementing ordered disjunction using answer set solvers for normal
  programs.
\newblock In {\em European Workshop on Logics in Artificial Intelligence}.
  Springer, 444--456.

\bibitem[\protect\citeauthoryear{Calimeri, Faber, Gebser, Ianni, Kaminski,
  Krennwallner, Leone, Ricca, and Schaub}{Calimeri
  et~al\mbox{.}}{2012}]{calimeri12aspcore2}
{\sc Calimeri, F.}, {\sc Faber, W.}, {\sc Gebser, M.}, {\sc Ianni, G.}, {\sc
  Kaminski, R.}, {\sc Krennwallner, T.}, {\sc Leone, N.}, {\sc Ricca, F.}, {\sc
  and} {\sc Schaub, T.} 2012.
\newblock {A}{S}{P}-{C}ore-2: {I}nput language format.
\newblock {\em ASP Standardization Working Group, Tech. Rep\/}.

\bibitem[\protect\citeauthoryear{Delgrande, Schaub, and Tompits}{Delgrande
  et~al\mbox{.}}{2003}]{delgrande03aframework}
{\sc Delgrande, J.~P.}, {\sc Schaub, T.}, {\sc and} {\sc Tompits, H.} 2003.
\newblock A framework for compiling preferences in logic programs.
\newblock {\em Theory and Practice of Logic Programming\/}~{\em 3,\/}~2,
  129--187.

\bibitem[\protect\citeauthoryear{Ferraris}{Ferraris}{2011}]{ferraris11logic}
{\sc Ferraris, P.} 2011.
\newblock Logic programs with propositional connectives and aggregates.
\newblock {\em ACM Transactions on Computational Logic (TOCL)\/}~{\em 12,\/}~4,
  25.

\bibitem[\protect\citeauthoryear{Ferraris, Lee, Lifschitz, and Palla}{Ferraris
  et~al\mbox{.}}{2009}]{ferr09b}
{\sc Ferraris, P.}, {\sc Lee, J.}, {\sc Lifschitz, V.}, {\sc and} {\sc Palla,
  R.} 2009.
\newblock Symmetric splitting in the general theory of stable models.
\newblock In {\em Proceedings of International Joint Conference on Artificial
  Intelligence (IJCAI)}. 797--803.

\bibitem[\protect\citeauthoryear{Gelfond and Lifschitz}{Gelfond and
  Lifschitz}{1991}]{gel91b}
{\sc Gelfond, M.} {\sc and} {\sc Lifschitz, V.} 1991.
\newblock Classical negation in logic programs and disjunctive databases.
\newblock {\em New Generation Computing\/}~{\em 9}, 365--385.

\end{thebibliography}
